\documentclass{article}
\usepackage{graphicx}
\usepackage{tikz}
\usepackage{tabularx}

\usepackage{float}

\usepackage[none]{hyphenat}

\usepackage{natbib}
\usepackage{quotes}

\usepackage{algorithm2e}
\SetKwInput{KwData}{Input}
\SetKwInput{KwResult}{Output}

\newcommand{\w}{\mathbf{w}} 
\newcommand{\Pa}{\mathrm{Pa}}   
\newcommand{\M}{\mathcal{M}}  
\newcommand{\V}{\mathcal{V}}  
\newcommand{\E}{\mathcal{E}}  
\newcommand{\G}{\mathcal{G}}  
\newcommand{\D}{\mathcal{D}}  
\newcommand{\F}{\mathcal{F}}  
\newcommand{\X}{\mathcal{X}}  
\newcommand{\Y}{\mathcal{Y}}  
\newcommand{\C}{\mathcal{C}}  

\newcommand{\dict}[1]{\texttt{\{#1\}}}

\usepackage[title]{appendix}

\usepackage{amssymb}
\usepackage{amsthm}
\usepackage{amsmath}
\theoremstyle{definition}
\newtheorem{example}{Example}[section]
\newtheorem{theorem}{Theorem}[section]
\newtheorem{corollary}{Corollary}[theorem]

\begin{document}

\title{Reducing Causality to Functions with Structural Models}
\author{Tianyi Miao\\
University of Pennsylvania\\
\texttt{mtianyi@sas.upenn.edu}
}
\date{}
\maketitle

\begin{abstract}
The precise definition of causality is currently an open problem in philosophy and statistics. We believe causality should be defined as functions (in mathematics) that map causes to effects. We propose a reductive definition of causality based on Structural Functional Model (SFM). Using delta compression and contrastive forward inference, SFM can produce causal utterances like "X causes Y" and "X is the cause of Y" that match our intuitions. We compile a dataset of causal scenarios and use SFM in all of them. SFM is compatible with but not reducible to probability theory. We also compare SFM with other theories of causation and apply SFM to downstream problems like free will, causal explanation, and mental causation.

\end{abstract}

\section{Introduction}
What is causation? What does it mean to say one thing causes another? Is it possible to define causation in non-causal terms?

We can easily find examples where "correlation doesn't imply causation." Ice cream sales are positively correlated with deaths by drowning, but ice cream doesn't cause drowning. However, this doesn't tell us what causation really is. While probabilistic independence and correlation coefficients have clear mathematical definitions, the precise definition of causality remains a subject of ongoing debate.

Embracing a functional theory of causation, we argue that causality essentially \textit{is} functions that map causes to effects.
While functions are distinct from probability theory and sufficiently general for scientific purposes, we can place additional constraints and formalize \textit{Structural Functional Model} (SFM), which better fit intuitions in \textit{causal utterances}:
\begin{enumerate}
\item Forward inference from causes to effects:
\begin{itemize}
\item What if X? Y.
\item Had it been X, it would have been Y.
\end{itemize}
\item \textit{Actual causality} (separating "actual causes" from background conditions):
\begin{itemize}
\item X causes/doesn't cause Y.
\item X is/isn't the cause of Y.
\item What is the cause of Y? X.
\end{itemize}
\end{enumerate}

Throughout this paper, the word "function" exclusively denotes a mathematical function (Appendix \ref{math_foundations}). We'll never use it to mean "intended purpose or task" as in "the functions of cellphones include texting." The word "functional" is only used as the adjective form of "function."

For SFM, we'll explicitly separate its representation, inference, and learning \citep{koller2009pgm}:
\begin{itemize}
\item \textit{Representation} is the declarative model of "what the world is like."
\item \textit{Inference} assumes the representation is correct and answers queries regarding particular instances, such as computing values of unknown variables given known variables.
\item \textit{Learning} inductively constructs a representation from empirical data.
\end{itemize}
Such decoupling allows us to design general-purpose inference and learning algorithms that work for different task-specific representations.

\section{Representation: A Roadmap}\label{roadmap}
In this section, we build the representation of SFM by incrementally adding functions, directed graphs, composition, contrast, and delta compression into a unified model. Each additional component will help SFM better fit intuitions about causal utterances, sometimes at the cost of generality.

Motivated by theoretical and pragmatic benefits like simplicity, expressiveness, and computational efficiency, the definition of SFM is unambiguous, mathematical, and reductive. It contains no circular definition because it doesn't rely on causal concepts like intervention and agency.

\subsection{Causal Relata}\label{roadmap:relata}
When we say "$X$ causes $Y$", what kinds of things are $X$ and $Y$? How do we represent a world? Classifying by causal relata, there are 4 kinds of causal relationships \citep{sep-causation-metaphysics}:
\begin{itemize}
\item \textbf{Token causation}: I frequently water my flower in my garden, causing it to grow tall.
\item \textbf{Type causation}: Watering a plant frequently causes it to grow tall.
\item \textbf{Token influence}: How much I water my flower in my garden influences how tall it grows.
\item \textbf{Type influence}: How much a plant is watered influences how tall it grows.
\end{itemize}
Influence relates variables (a variable can have one of many values); causation relates values of variables.

Tokens are specific; types are general. Since this type-token distinction applies to non-causal models too, it's not central to causality. SFM doesn't endorse any particular theory of physics or metaphysics, so it's up to the user to specify how variables correspond to real-world things.

Formally, let $\V$ be a set of \textit{nodes} (we use "nodes" instead of "variables" to avoid confusion with random variables) and $\D$ be a function that maps nodes to their domains. For node $u \in \V$, its \textit{domain} $\D[u]$ is the set of values it can take on. An \textit{assignment} is a function that maps each node to a value in its domain.
\begin{itemize}
\item A \textit{complete assignment} $\w: \V \to \bigcup_{u \in \V} \D[u]$ assigns values to all nodes, satisfying $\forall u\in\V: \w(u)\in\D[u]$.
\item A \textit{partial assignment} $\w_{|\X}: \X \to \bigcup_{u \in \X} \D[u]$ assigns values to a subset $\X \subseteq \V$ of nodes, satisfying $\forall u\in\X: \w_{|\X}(u)\in\D[u]$.
\item $\w_{|\X} \subseteq \w$ iff $\forall u\in\X: \w_{|\X}(u)=\w(u)$.
\end{itemize}

We use dictionary notations $\dict{node1:value1, node2:value2, \dots}$ for assignments (and discrete finite functions in general). Nodes, values, and assignments are different things.
Influence relates nodes (\texttt{Water} influences \texttt{Growth}), while causation relates assignments ($\dict{Water:High}$ causes $\dict{Growth:Tall}$).

\begin{enumerate}
\item The set of all complete assignments forms the Cartesian product $\prod_{u \in \V} \D[u]$.
\item A \textit{team} $R$ is a set of complete assignments \citep{vaananen2007dependence}, so $R \subseteq \prod_{u \in \V} \D[u]$. $R$ is a relation.
\item For a modal/counterfactual/possible-world interpretation, each complete assignment $\w$ is a world.
Each node is a feature/property/aspect/variable of the world.
$R$ is the set of possible worlds; $(\prod_{u \in \V} \D[u]) \setminus R$ is the set of impossible worlds.
\item For a database interpretation, each $\w$ is an individual/person/record/item. $R$ is a population containing many individuals. Each node is a property/attribute/feature of that individual.
\item A complete assignment $\w$ \textit{satisfies} team $R$ iff $\w \in R$.
\item A team $R$ is \textit{satisfiable} iff $R$ is nonempty. $R$ is unsatisfiable iff $R = \emptyset$.
\item If any domain $\D[u]$ is empty, the Cartesian product $\prod_{u \in \V} \D[u]$ is empty and there's no satisfiable $R$, so we'll only consider nonempty domains.
\item An assignment $\w_{|\X}$ is \textit{permitted} by $R$ iff $\exists \w \in R: \w \supseteq \w_{|\X}$.
We call this $\w$ an \textit{induced complete assignment} of $\w_{|\X}$.
\item Partial assignments $\w_{|\X_1}, \w_{|\X_2}, \dots, \w_{|\X_k}$ are \textit{compatible} with each other iff $\exists \w \in R: \forall i \in \{1, 2, \dots, k\}: \w \supseteq \w_{|\X_i}$.
\end{enumerate}

\textit{We will say "$\X$ influences $\Y$" and "$\w_{|\X}$ causes $\w_{|\Y}$," where $\X$ and $\Y$ are sets of nodes; $\w_{|\X}$ and $\w_{|\Y}$ are partial assignments.}

\subsection{A Functional Theory of Causation}\label{roadmap:function}
Many causal scenarios are not reducible to probability theory. For example, flipping the light switch turns on the light, but doesn't affect the TV. This system of electric circuits is deterministic and fully-specified. We can consistently predict the "independence" between light switch and TV and what would happen given the switches' status, using functions alone without probabilities.

According to the \textit{functional theory of causality}, causality essentially \textit{is} mathematical functions (left-total, right-unique relations) that map causes to effects.
\citet{russell1912notion} briefly mentions that the cause (functionally) determines the effect.
\citet{simon1966cause} explicitly defend that causation is "a function of one variable (the cause) on to another (the effect)."
Structural Causal Model (SCM) \citep{pearl2009causality} uses multi-input single-output functions in structural equations to represent "laws" or "mechanisms" of the world.

"Causality as functions" becomes immediately obvious once it's pointed out. For example,
\begin{enumerate}
\item In $y = f(x)$, we call $x$ the independent variable and $y$ the dependent variable, like how effects depend on causes.
\item Describing "rain influences wheat growth" with $\texttt{WheatGrowth} = f(\texttt{Rain})$, the input-output mappings are:
\begin{itemize}
\item With no rain, wheat doesn't grow.
\item With moderate rain, wheat grows moderately.
\item With heavy rain, wheat grows very well.
\end{itemize}
\item The light-switch-and-TV example can be described by  $\texttt{Light}=f_1(\texttt{LightSwitch})$ and $\texttt{TV}=f_2(\texttt{TVSwitch})$.
\end{enumerate}

Two key properties distinguish functions from other kinds of relations:
\begin{enumerate}
\item \textbf{Right-uniqueness}: 1 input value cannot simultaneously associate with 2 or more different output values. Functions can only be many-to-one or one-to-one, never one-to-many.
This explains why causes "necessitate" or "are sufficient for" their effects (given the underlying function).
\item \textbf{(Possible) non-injectiveness}: Some functions can map different input values to the same output value, like $y=x^2$ over real numbers. Non-injective functions cannot be inverted. This explains the asymmetry of causation: different causes can lead to the same effect.
\end{enumerate}

\textit{Functional dependencies} are properties of a team $R \subseteq \prod_{u \in \V} \D[u]$: For $\X, \Y \subseteq \V$,
\begin{enumerate}
\item \textbf{Value-level dependency}: We say "$\Y$ functionally depends on $\w_{|\X}$" ($\w_{|\X} \xrightarrow{.} \Y$) or "$\w_{|\Y}$ functionally depends on $\w_{|\X}$" ($\w_{|\X} \xrightarrow{.} \w_{|\Y}$) when given $\w_{|\X}$, there exists exactly one $\w_{|\Y}$ that's compatible with $\w_{|\X}$.
\item \textbf{Node-level dependency}: We say "$\Y$ functionally depends on $\X$" ($\X \xrightarrow{.} \Y$) when $\w_{|\X} \xrightarrow{.} \Y$ for every permitted $\w_{|\X}$.
\item Value-level and node-level dependencies can be different. In $\w(Y) = \w(X_1) \lor \w(X_2) \lor \w(X_3)$, value-level $\{X_1: 1\} \xrightarrow{.} \{Y: 1\}$ is true; node-level $\{X_1\} \xrightarrow{.} \{Y\}$ is false; node-level $\{X_1, X_2, X_3\} \xrightarrow{.} \{Y\}$ is true.
\end{enumerate}

Node-level functional dependency $\X \xrightarrow{.} \Y$ satisfies right-uniqueness: $\forall \w_1, \w_2 \in R: (\w_{1|\X} = \w_{2|\X}) \Rightarrow (\w_{1|\Y} = \w_{2|\Y})$.
So there's a function $f: \{\w_{|\X} \ | \exists \w\in R: \w\supseteq\w_{|\X}\} \to \{\w_{|\Y}\ |\exists \w \in R: \w\supseteq\w_{|\Y}\}$ such that $\forall \w \in R: \w_{|\Y} = f(\w_{|\X})$.
We thus define \textit{functional determination}:
\begin{enumerate}
\item \textbf{Node-level determination}: We say "$\X$ functionally determines $\Y$ via $f$" ($\X \xrightarrow{f} \Y$) when $\forall \w \in R: \w_{|\Y} = f(\w_{|\X})$.
\item \textbf{Value-level determination}: We say "$\w_{|\X}$ functionally determines $\w_{|\Y}$ via $f$" ($\w_{|\X} \xrightarrow{f} \w_{|\Y}$) when $\X \xrightarrow{f} \Y$ and $\w_{|\Y} = f(\w_{|\X})$.
\end{enumerate}
In compliance with conventions from dependence logic \citep{sep-logic-dependence} and relational databases \citep{database-2020}, functional dependency $\X \xrightarrow{.} \Y$ doesn't contain $f$, while our functional determination $\X \xrightarrow{f} \Y$ does.

Influence is node-level functional determination; causation is value-level functional determination. In $\w_{|\Y} = f(\w_{|\X})$, $\w_{|\X}$ is the cause, $\w_{|\Y}$ is the effect, and $f$ is an underlying mechanism/law-of-nature (since $\w_{|\Y} = f(\w_{|\X})$ is true in every possible world $\w \in R$).

Generally, causality is the study of functional dependency (e.g. Armstrong's Axioms), functional determination, and relational independence \citep{gradel2013dependence}. It's nontrivial because these concepts cannot be reduced to probability theory.

\textit{We say "$\w_{|\X}$ causes $\w_{|\Y}$" when $\X \xrightarrow{f} \Y$ and $\w_{|\Y}=f(\w_{|\X})$. We say "$\X$ influences $\Y$" when $\X \xrightarrow{.} \Y$.}

\subsection{Directed Graphs}\label{roadmap:directed_graphs}
Previously, we first have a team $R$ and then find functional determinations as properties of $R$. Now we take the opposite direction. We start with a set of functional determinations $\textbf{FDet} = \{\X_1 \xrightarrow{f_1} \Y_1, \X_2 \xrightarrow{f_2} \Y_2, \dots \X_n \xrightarrow{f_n} \Y_n\}$, which then select $R_{\textbf{FDet}} \subseteq \prod_{u \in \V} \D[u]$ as all $\w$ that satisfies \textbf{FDet}. Here "all" is necessary for defining a unique $R_{\textbf{FDet}}$, because functional dependencies and determinations are downward-closed (if $R_1$ satisfies \textbf{FDet}, then any subset $R_2 \subseteq R_1$ also satisfies \textbf{FDet} \citep{sep-logic-dependence}).

When we draw diagrams to illustrate causal relationships, we want arrows to point from causes to effects.
Structural Causal Model (SCM) \citep{pearl2009causality, pearl2009overview, halpern2005causes1} generalizes this intuition, subsumes the graphical and potential-outcome frameworks, and is the most popular causal model in statistics, econometrics, and epidemiology. Our SFM inherits the following ideas from SCM:
\begin{enumerate}
\item A causal system is represented as a (usually finite and acyclic) directed graph.
\item One mechanism's effect can be another mechanism's cause. One function's output can be another function's input.
\item A node's value is functionally determined by the values of its parents.
\item Unlike SCM, our SFM doesn't use "intervention" in its definition at all (Section \ref{compare:scm}).
\end{enumerate}

Besides nodes $\V$ and domains $\D$, an SFM $\M = (\V, \E, \D, \F)$ also has:
\begin{enumerate}
\item $\E \subseteq \V \times \V$ is a set of directed edges.
\item In a directed graph $\G = (\V, \E)$, a node $u$ is \textit{exogenous} (exo-node $u \in \V_{exo}$) iff it's a root node; otherwise, it's \textit{endogenous} (endo-node $u \in \V_{endo}$).
\item We write exo-assignment $\w_{|\V_{exo}}$ as $\w_{exo}$ and endo-assignment $\w_{|\V_{endo}}$ as $\w_{endo}$.
\item $\F$ maps every endo-node $u \in \V_{endo}$ to exactly one \textit{structural function} $\F[u]: (\prod_{p \in \Pa(u)} \D[p]) \to \D[u]$.
\item $\F[u]: \w_{|\Pa(u)} \mapsto \w(u)$ maps an assignment over $u$'s parents to a value of $u$.
\item $R_{\M}=\{\w \in \prod_{u \in \V} \D[u] \ |\ \forall u \in \V_{endo}: \w(u) = \F[u](\w_{|\Pa(u)})\}$ is the set of all complete assignments satisfying $\M$.

Equivalently, $\M$ specifies functional determinations $\textbf{FDet}_{\M} = \{\Pa(u) \xrightarrow{f_u} \{u\} \}_{u \in \V_{endo}}$, where $f_u(\w_{|\Pa(u)})=\{u:\F[u](\w_{|\Pa(u)})\}$.
\end{enumerate}

\begin{example}\label{example:abcd}
Consider SFM $\M = (\V, \E, \D, \F)$:
\begin{itemize}
\item $\V = \{A, B, C, D, E\}$
\item $\E = \{(A, B), (B, D), (C, D), (C, E)\}$
\item $\D = \{A: \mathbb{R}, B: \mathbb{R}, C: \mathbb{R}, D: \mathbb{R}, E: \mathbb{R}\}$
\item For simplicity, we'll abuse notations and write $\F[u](\w_{|\Pa(u)})$ as $\F[u](\w)$:

$\F[B](\w)=\w(A)^2\\ \F[D](\w)=\w(B)+\w(C)\\ \F[E](\w)=\w(C)\times 7$
\end{itemize}
\begin{figure}[H]\label{fig:abcde_full}
\centering
\begin{tikzpicture}
\node[shape=circle,draw=black,style=thick] (A) at (0,2.5) {A};
\node[shape=circle,draw=black,style=thick] (B) at (0,1.25) {B};
\node[shape=circle,draw=black,style=thick] (C) at (2,1.25) {C};
\node[shape=circle,draw=black,style=thick] (D) at (1,0) {D};
\node[shape=circle,draw=black,style=thick] (E) at (3,0) {E};

\path [->, style=thick] (A) edge node[left] {$\F[B](\w) = \w(A)^2$} (B);
\path [->, style=thick] (B) edge node[left] {$\F[D](\w)=\w(B)+\w(C)$} (D);
\path [->, style=thick] (C) edge node[left] {} (D);
\path [->, style=thick] (C) edge node[right] {$\F[E](\w)=\w(C)\times 7$} (E);
\end{tikzpicture}

\caption{A simple finite acyclic SFM.}
\end{figure}
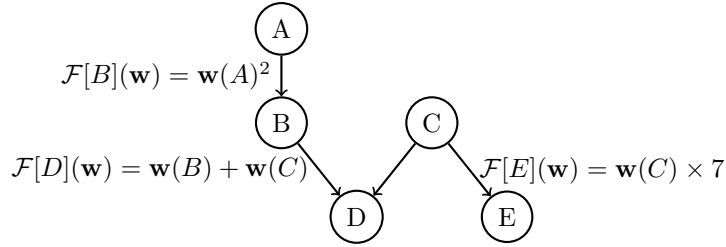

\begin{itemize}
\item $A, C \in \V_{exo}$ are exo-nodes; $B, D, E \in \V_{endo}$ are endo-nodes.
\item $A \to B \to D$ forms a causal chain, $B \to D \leftarrow C$ forms a "common effect" structure, and $D \leftarrow C \to E$ forms a "common cause" structure.
\item $\{A: i, B: -1, C: 10, D: 9, E: 70\}$ isn't an assignment over $(\V, \D)$, because the complex number $i \notin \mathbb{R}$ is outside of $A$'s domain.
\item $\{A: 2, B: 2, C: 2, D: 2, E: 2\}$ is a complete assignment over $(\V, \D)$, but it doesn't satisfy $\M$.
\item $\{A: 3, B: 9, C: -\pi, D: 9-\pi, E: -7\pi\}$ is a complete assignment that satisfies $\F$, so $\M$ is satisfiable.
\item Therefore, partial assignments $\{A: 3, B: 9\}$ and $\{D: 9-\pi, E: -7\pi\}$ are permitted and compatible with each other.
\item $\{D: -10, E: 7\}$ isn't permitted because no $\w \in R_{\M}$ extends it.
\end{itemize}
\end{example}

Some design choices of SFM inevitably restrict the kinds of functional dependencies that we can talk about:
\begin{enumerate}
\item For simplicity, we only consider finite nodes because no important application requires an infinite SFM.

\item Not every set of functional determinations be covered (entailed) by an SFM, even if we allow cycles.

Consider $\V = \{X, Y, Z\}$ with real-valued domains, the team $R_1 = \{\w |\  \w(X)^2 = \w(Y) = \w(Z)^2\}$ has functional determinations $\{X\} \xrightarrow{\w(Y)=\w(X)^2} \{Y\}$ and $\{Z\} \xrightarrow{\w(Y)=\w(Z)^2} \{Y\}$. There's no SFM $\M$ with $R_{\M}=R_1$.

Generally, SFM cannot represent one node being functionally determined by multiple "separate" functions/mechanisms, each individually sufficient for its value. This differs from symmetric overdetermination (Section \ref{benchmark:overdetermination}), which is just multi-input Boolean OR.

\item The intersection of SFMs, however, can cover any set of functional determinations.

We say $\w$ satisfies the \textit{SFM-intersection} over $(\M_1, \M_2, \dots, \M_n)$ if $\w \in \bigcap_{i=1}^{n} R_{\M_i}$ ($\w$ satisfies every individual $\M_i$).
\begin{theorem}
For any set of functional determinations $\textbf{FDet}$ over finite $\V$, there exists a finite SFM-intersection that covers it.
\end{theorem}
\begin{proof}
Since $\V$ is finite, $\textbf{FDet}$ is finite.
For every $\X_i \xrightarrow{f_i} \Y_i$ in $\textbf{FDet}$, we construct $\M_i = (\V, \E_i, \D, \F_i)$ with edges $\E_i = \X_i \times \Y_i$ and structural functions $\F_i[y]: \w_{|\X_i} \mapsto f_i(\w_{|\X_i})(y)$ for $y \in \Y_i$. The SFM-intersection over all $\M_i$ entails $\textbf{FDet}$.
\end{proof}

An \textit{SFM-intersection-proper} is an SFM-intersection that cannot be entailed by an SFM.

Besides $\w(X)^2=\w(Y)=\w(Z)^2$, SFM-intersection-proper can express autonomous differential equations like $\frac{d}{dt} x(t) = f(x(t))$ while SFM cannot. The differential operator $\frac{d}{dt}$ is also a function, so we derive 2 functional determinations: $A \xrightarrow{\frac{d}{dt}} B$ and $A \xrightarrow{f} B$. Here $\{A: x(t), B: x'(t)\}$ is permitted iff $x'(t)=f(x(t))$.

\item Why do people dislike SFM-intersection?

It's nearly impossible to find an uncontrived, everyday causal system that's only describable by SFM-intersection-proper. \citet{kim2005physicalism} even explicitly formulates the \textit{Principle of Causal Exclusion} against "more than one sufficient cause" in this spirit.
This intuitive dislike is unjustified, but when taken as a primitive desideratum, it entails people's preference of some SFMs over others for modeling reality.

We suggest 2 possible reasons for disliking SFM-intersection-proper:
\begin{enumerate}
\item Intersection of multiple SFMs creates too much mental computational burden and people prefer simpler models.

In many cases (Section \ref{learn:flagpole}, \ref{apply:mental}), people dislike the very form of SFM-intersection, even though the underlying $R=R_{\M}$ can be modeled by some SFM $\M$.
\item SFM-intersection-proper suffers from the \textit{possibly-unsatisfiable-laws objection (PULO)}, which applies to any set of functional dependencies $\textbf{FDep}=\{\X_i \xrightarrow{.} \Y_i\}_{i=1}^{n}$ such that some $\{f_i\}_{i=1}^{n}$ makes $\textbf{FDet}=\{\X_i \xrightarrow{f_i} \Y_i\}_{i=1}^{n}$ unsatisfiable.

No world satisfies $\textbf{FDet}$, but our actual world exists, so we must reject $\textbf{FDet}$. PULO takes one unjustified step further, suggesting that $\textbf{FDep}$ should also be rejected, even if some other $\{g_i\}_{i=1}^{n}$ makes $\{\X_i \xrightarrow{g_i} \Y_i\}_{i=1}^{n}$ satisfiable, because $\textbf{FDep}$ "opens the gate" to unsatisfiable laws. From another perspective, PULO expresses a desire for guaranteed satisfiability under any function set.

For example, $\textbf{FDep} = \{\{X\} \xrightarrow{.} \{Y\}, \{Z\} \xrightarrow{.} \{Y\}\}$ suffers from PULO because $\textbf{FDet} = \{\{X\} \xrightarrow{\w(Y)=\w(X)^2} \{Y\}, \{Z\} \xrightarrow{\w(Y)=-\w(Z)^2-1} \{Y\}\}$ is unsatisfiable over real-valued domains.

\end{enumerate}

\item Why do we make SFM acyclic?

PULO strikes again: When there are self-loops or cycles in the graph, there exist function sets that make the SFM unsatisfiable, such as $A=A+1$ and $\{A=B+1; B=A+1\}$:

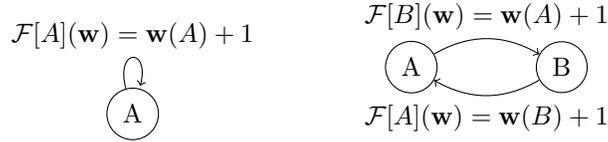
\begin{figure}[H]
  \centering
  \begin{tikzpicture}[node distance=2cm]
    \node[circle, draw] (A) {A};
    \path[->] (A) edge [loop above] node [above] {$\F[A](\w)=\w(A)+1$} (A);
  \end{tikzpicture}
  \hspace{1cm}
  \begin{tikzpicture}[node distance=2cm]
    \node[circle, draw] (A) {A};
    \node[circle, draw, right of=A] (B) {B};
    \path[->] (A) edge [bend left] node [above] {$\F[B](\w)=\w(A)+1$} (B)
              (B) edge [bend left] node [below] {$\F[A](\w)=\w(B)+1$} (A);
  \end{tikzpicture}
  \caption{SFM with a self-loop (left) and SFM with a cycle (right); their structural functions are not satisfiable over real-valued domains.}
\end{figure}

\item Besides simplicity and intuitive appeals, finite acyclic SFM $\M$ has other nice properties (Section \ref{infer:forward}):
\begin{itemize}
\item $\M$ is satisfiable for any $\F$.
\item $\V_{exo}$ functionally determines $\V_{endo}$ via $\w_{endo} \subseteq \w = \texttt{VFI}(\M, \w_{exo})$.
\end{itemize}
Are they worth the price of rejecting many (possibly satisfiable) sets of functional dependencies? We're unsure.

\item Different SFMs $\M_1 \ne \M_2$ over the same $(\V, \D)$ can be "semantically equivalent" $R_{\M_1}=R_{\M_2}$, which entails "$\w_{|\Y}=f(\w_{|\X})$ in $\M_1$ iff $\w_{|\Y}=f(\w_{|\X})$ in $\M_2$", including $\texttt{VFI}(\M_1, \w_{exo})=\texttt{VFI}(\M_2, \w_{exo})$ for all $\w_{exo}$.
\end{enumerate}

\textit{We'll only consider functional determinations that can be modeled by finite acyclic SFMs, where an endo-node is functionally determined by its parents.}

\subsection{Composition and Decomposition}\label{roadmap:compose}
Since $\w_{exo}$ functionally determines $\w_{endo}$ via $\w_{endo}\subseteq\texttt{VFI}(\M, \w_{exo})$ (Section \ref{infer:forward}), we produce all causal utterances as "$\w_{exo}$ causes $\w_{endo}$."

This syntax is simple, but an ostensible flaw is that only exo-assignments can be causes. In $A \to B \to C$, we cannot say "$\{B: b\}$ causes $\{C: c\}$" because $B$ is an endo-node. This problem is solved by considering the sub-SFM $B \to C$, where $B$ becomes an exo-node. Sub-SFM generalizes \citet{pearl2009causality}'s surgical intervention, which cuts off all incoming edges to the nodes under intervention.

$\M_{sub} = (\V_{sub}, \E_{sub}, \D_{sub}, \F_{sub})$ is a \textit{sub-SFM} of $\M = (\V, \E, \D, \F)$ when:
\begin{enumerate}
\item $(\V_{sub}, \E_{sub})$ is a subgraph of $(\V, \E)$, i.e. $\V_{sub} \subseteq \V$, $\E_{sub} \subseteq \E$, and $(u, v) \in \E_{sub} \Rightarrow (u \in \V_{sub}) \land (v \in \V_{sub})$.
\item $\forall u \in \V_{sub|endo}: \Pa_{sub}(u)=\Pa(u)$.

\item $\forall u \in \V_{sub}: \D_{sub}[u] = \D[u]$
\item $\forall u \in \V_{sub|endo}: \F_{sub}[u] = \F[u]$
\end{enumerate}
An exo-node in $\M$ can be nonexistent or exogenous in $\M_{sub}$; an endo-node in $\M$ can be nonexistent, exogenous, or endogenous (with the same parents and structural function) in $\M_{sub}$, so the mechanisms-of-nature are preserved.

We can \textit{compose} a set of smaller SFMs $\{\M_1, \M_2, \dots, \M_m\}$ into a bigger SFM $\M$ without altering any structural function, if the following prerequisites are met for any pair of $(\M_i, \M_j)$:
\begin{enumerate}
\item $\forall u \in \V_i \cap \V_j: \D_i[u]=\D_j[u]$
\item $\forall u \in \V_{i|endo} \cap \V_{j|endo}: \Pa_i(u) = \Pa_j(u)$
\item $\forall u \in \V_{i|endo} \cap \V_{j|endo}: \F_i(u) = \F_j(u)$
\end{enumerate}
These prerequisites ensure that the composition $\M = (\bigcup_{i=1}^m \V_i, \bigcup_{i=1}^m \E_i, \bigcup_{i=1}^m \D_i, \bigcup_{i=1}^m \F_i)$ is well-defined. For $\bigcup_{i=1}^m \D_i$ and $\bigcup_{i=1}^m \F_i$,
\begin{enumerate}
\item $\D_i$ maps nodes to domains.
\item $\F_i$ maps nodes to structural functions.
\item Functions (including $\D$ and $\F$) are binary relations.
\item The union of sets/relations/functions is well defined.
\item The prerequisites ensure that each node $u$ has exactly one unique $\D[u]$ and at most one unique $\F[u]$ across all $i$, so $\bigcup_{i=1}^m \D_i$ and $\bigcup_{i=1}^m \F_i$ are right-unique and thus functions.
\end{enumerate}

The \textit{decomposition} of SFM $\M$ is a set of sub-SFMs $\{\M_1, \M_2, \dots, \M_m\}$ that can compose into $\M$. While composition of sub-SFMs (when allowed) is unique, there can be multiple different decompositions of an SFM, the most trivial being "keeping the original SFM itself" and the most fragmented being "one sub-SFM for each endo-node and its parents."

Composition shows how small, local, and simple sub-mechanisms can be pieced together into one big, global, and complex system, while decomposition breaks down a large system into small sub-mechanisms. Therefore, we can deductively reason about a big, unrepeatable event using its components and their interconnections.

\textit{With composition-decomposition, we can say "$\w_{exo}$ causes $\w_{endo}$" relative to some sub-SFM.}

\subsection{Contrastive Causation}\label{roadmap:contrast}
Currently, SFM can already perfectly express a causal system by correctly answering all "what's $\w_{endo}$ if $\w_{exo}$" questions. But in causal utterances, people only say "the actual causes" and omit background conditions (Section \ref{benchmark:background}).
The selection of actual causes takes 2 steps: contrast and omission. We'll discuss contrast in this section.

\citet{schaffer2005contrastive} believes causation is contrastive. Besides the 2-argument surface form (cause, effect), the 4-argument underlying form includes contrast on both sides:
\begin{itemize}
\item Surface form: Pam's throwing the rock caused the window to shatter.
\item Contrastive form 1: Throwing the rock (rather than the pebble) caused the window to shatter (rather than crack).
\item Contrastive form 2: Throwing the rock (rather than not throwing it) caused the window to shatter (rather than remain intact).
\end{itemize}

We specify 2 assignments $\w_a, \w_c$ for contrastive causal utterance "$\w_{a|exo}$ (rather than $\w_{c|exo}$) causes $\w_{a|endo}$ (rather than $\w_{c|endo}$)":
\begin{enumerate}
\item \textit{Actual assignment} $\w_{a}$ corresponds to the actual world (i.e. what actually happens).
\item \textit{Contrastive assignment} $\w_{c}$ is selected using one of two heuristics:
\begin{enumerate}
\item $\w_c$ is a default/expected/normal/typical world; $\w_a$ is an anomalous/unexpected deviation from the default.

Normality inevitably comes with value judgments, but contrast reduces "finding the actual causes" to "finding a default world," which is a nontrivial simplification.
\item With $\w_{a}$ available first, we tweak $\w_{a|exo}$ into $\w_{c|exo}$ by changing the values of a few exo-nodes of interest. We then obtain $\w_c = \texttt{VFI}(\M, \w_{c|exo})=\texttt{CFI}(\M, \w_a, \w_{c|exo})$ through forward inference (Section \ref{infer:forward}).

This is common when too many nodes in $\w_a$ have non-default values, or when there's no appropriate default world.
\end{enumerate}
\end{enumerate}
Our contrastive causation is slightly simpler than \citet{schaffer2005contrastive}'s and \citet{halpern2015graded}'s, because we only need to specify one contrastive world $\w_c$ (rather than many).

Contrast is common in our causal intuition:
\begin{enumerate}
\item People often characterize causality as "changing the cause will also change the effect" or "making a difference." Ignoring the manipulation aspect of an agent changing an object, change is inherently contrastive - there's an old state that changes to a new state.
\item Some philosophers try to define "event $X$ causes event $Y$" as "$X$ raises the probability of $Y$ ($\Pr[Y|X] > \Pr[Y|\lnot X]$)." This definition fails to address causal asymmetry and spurious correlations \citep{sep-causation-probabilistic}, so it's never popular among statisticians. However, the very idea of "raising" contains a contrast between a world with $X$ and a world with $\lnot X$.
\item The contrast of treatment effects is formalized in statistical causal inference. Using the potential outcome notations in \citet{hernan2020causal},
\begin{itemize}
\item causal risk difference: $\Pr[Y^{a=1} = 1] - \Pr[Y^{a=0} = 1]$
\item causal risk ratio: $\frac{\Pr[Y^{a=1} = 1]}{\Pr[Y^{a=0} = 1]}$
\item causal odds ratio: $\frac{\Pr[Y^{a=1} = 1] / \Pr[Y^{a=1} = 0]}{\Pr[Y^{a=0} = 1] / \Pr[Y^{a=0} = 0]}$
\end{itemize}
These measurements all involve a contrast between random variables $Y^{a=0}$ (effect under treatment 0) and $Y^{a=1}$ (effect under treatment 1).
\item To understand a function $y=f(x)$, we often record an initial input value $x_0$ and its corresponding output value $y_0=f(x_0)$; we then change $x_0$ to $x_1$  and see how the output value $y$ changes in response. For example, derivatives in calculus help quantify how "sensitive" the output is with respect to the input.
\end{enumerate}

\textit{With actual assignment $\w_a$ and contrastive assignment $\w_c$, we say "$\w_{a|exo}$ (rather than $\w_{c|exo}$) causes $\w_{a|endo}$ (rather than $\w_{c|endo}$)."}

\subsection{Delta Compression}\label{roadmap:delta}
To characterize omission in causal utterances, we consider $\C = \{u \in \V | \w_a(u) \ne \w_c(u)\}$: the nodes that have different values in $\w_a$ and $\w_c$. $|\C|$ is the \textit{Hamming distance} between $\w_a$ and $\w_c$. With $\C_{exo}= \C \cap \V_{exo}$ and $\C_{endo}= \C \cap \V_{endo}$, the final causal utterance is "$\w_{a|\C_{exo}}$ causes $\w_{a|\C_{endo}}$."

If something doesn't change, we don't mention it. We only mention the new values of changed nodes. This is an example of \textit{delta compression} \citep{suel2019delta}:

\textit{Encoder} wants to transmit a \textit{target file} to \textit{Decoder}. Encoder and Decoder can both access a \textit{reference file}. The target file is only slightly different from the reference file, so their \textit{delta} (change/difference) is much smaller than the target file itself. To reduce the amount of transferred data, Encoder computes the delta (using target and reference files) and sends it to Decoder; Decoder reconstructs the target file by adding the delta to the reference file.

Delta compression is widely used in version control, where we want to store many successive versions of the same file, but any 2 consecutive versions differ only slightly.

\begin{example}
Consider nodes $\{A, B, C, D\}$ with integer domains and assignments $\w_0, \w_1$:
\begin{enumerate}
\item $\w_0 = \{A:1,B:2,C:3,D:4\}$
\item $\w_1 = \{A:1,B:7,C:3,D:5\}$
\item $\C = \{B, D\}$
\item $\w_{0|\C} = \{B:2,D:4\}$
\item $\w_{1|\C} = \{B:7,D:5\}$
\end{enumerate}
\end{example}

People may prefer delta compression because it shortens causal utterances without losing information or introducing ambiguities. This saving of "mental bandwidth" is especially prominent when:
\begin{enumerate}
\item We want to represent many $\w_1$ relative to one $\w_0$.
\item Each $\w_1$ differs only slightly from $\w_0$, i.e. $|\C|$ is small relative to $|\V|$.
\end{enumerate}
In default-actual contrasts, the default $\w_c$ is kept constant for reference; in actual-tweaked contrasts, $\w_a$ is held for reference.

\textit{We say $\w_{a|\C_{exo}}$ causes $\w_{a|\C_{endo}}$, where $\C = \{u \in \V | \w_a(u) \ne \w_c(u)\}$ is the set of changed nodes.}

\section{Inference}\label{infer:section}
\subsection{Constraint Satisfaction}\label{infer:constraint}
During \textit{inference}, we assume the SFM $\M$ is true. An inference algorithm takes in assignment $\w_{|\X}$ over known nodes $\X \subseteq \V$ and a set of target nodes $\Y$, whose values we're interested in inferring. It then checks whether $\w_{|\X}$ is permitted and if so, returns one or more $\w_{|\Y}$ that's compatible with $\w_{|\X}$.

If all domains are finite, we can formulate SFM inference as a constraint satisfaction problem (CSP) \citep{russell2010artificial} and use off-the-shelf CSP solvers for inference:
\begin{enumerate}
\item The domain of node $u \in \V$ is $\D[u]$.
\item For each $u \in \V_{endo}$, its structural equation gives a $(|\Pa(u)|+1)$-ary constraint $\w(u) = \F[u](\w_{|\Pa(u)})$ over scope $\Pa(u) \cup \{u\}$.
\item Each known value $w_u=\w_{|\X}(u)$ for $u \in \X$ is a unary constraint $\w(u)=w_u$ over scope $\{u\}$.
\end{enumerate}

CSP does have a few drawbacks:
\begin{enumerate}
\item It's NP-complete in general.
\item It's unnecessary for most thought experiments, where the SFMs are small and solvable by hand.
\item It offers no guarantee for the existence or uniqueness of $\w_{|\Y}$. For example, if $y=f(x)$ is non-injective, different $x$ can be compatible with the same $y$.
\end{enumerate}

Thanks to right-uniqueness, inferring effects from causes is much easier.

\subsection{Forward Inference}\label{infer:forward}
\textit{Forward inference} infers effects from causes.
Given SFM $\M$, \textit{vanilla forward inference} (VFI) computes $\w = \texttt{VFI}(\M, \w_{exo})$, where $\w \supseteq \w_{exo}$ and of $\w \in R_{\M}$.

When $\G=(\V, \E)$ is finite and acyclic (and $\forall u \in \V: \D[u] \ne \emptyset$),
\begin{itemize}
\item $\M=(\V, \E, \D, \F)$ is satisfiable for any $\F$;
\item $\w_{exo}$ functionally determines $\w_{endo}$; \texttt{VFI} itself is a function.
\end{itemize}
Intuitively, \texttt{VFI} deterministically infers all effects given the root causes and mechanisms-of-nature.

\begin{theorem}\label{theorem:forward_infer}
\textbf{(Forward Inference)} In a finite acyclic SFM $\M$ (with nonempty domains), for any exo-assignment $\w_{exo}$, there exists a unique complete assignment $\w$ satisfying $\w \supseteq \w_{exo}$ and $\w \in R_{\M}$.
\end{theorem}
\begin{proof}

\begin{itemize}
\item \textbf{Existence}: Because $\G$ is finite, $\V$ is finite. Because $\G$ is acyclic, there exists a \textit{topological order} $L$ of nodes: an ordered list of all nodes such that $[(L[i], L[j]) \in \E] \Rightarrow [i < j]$. Using \textit{topological sort} algorithms like depth-first-search and Kahn's algorithm, we can compute $L$ in $\Theta(|\V|+|\E|)$ time; cycle detection is done simultaneously \citep{cormen2022clrs}.

Given $\w_{exo}$, we compute $\w_1$ sequentially from $i=1$ to $i=|\V|$ inclusive:
\begin{enumerate}
\item If $L[i] \in \V_{exo}$, we assign $\w_{1}(L[i]) \leftarrow \w_{exo}(L[i])$.
\item If $L[i] \in \V_{endo}$, we assign $\w_1(L[i]) \leftarrow \F[L[i]](\w_{1|\Pa(L[i])})$.
\item Any parent $L[j] \in \Pa(L[i])$ must appear earlier ($j < i$) than its child $L[i]$ because $L$ is a topological order. $\w_1(L[j])$ must have already been assigned, so $\w_{1|\Pa(L[i])}$ is well-defined.
\end{enumerate}
Because $\w_1(L[i])$ isn't modified after iteration $i$:
\begin{enumerate}
\item If $L[i] \in \V_{exo}$, $\w_1(L[i])=\w_{exo}(L[i])$ is always satisfied.
\item If $L[i] \in \V_{endo}$, $\w_1(L[i]) = \F[L[i]](\w_{1|\Pa(L[i])})$ is always satisfied.
\end{enumerate}
Therefore, $\w_1 \supseteq \w_{exo}$ and $\w_1$ satisfies $\M$.

This constructive proof also specifies the algorithm $\w=\texttt{VFI}(\M, \w_{exo})$, assuming every structural function $\F[u]$ is computable.

\item \textbf{Uniqueness}: Proving by contradiction, suppose instead that there's another $\w_2 \ne \w_1$ satisfying $\w_2 \supseteq \w_{exo}$ and $\w_2 \in R_{\M}$. With topological order $L$, there exists a smallest integer $i$ such that $\w_1(L[i]) \ne \w_2(L[i])$.

Because $L$ is a topological order, every parent $L[j] \in \Pa(L[i])$ appears earlier ($j < i$). Since $L[i]$ is the earliest node with different values, $\forall L[j] \in \Pa(L[i]): \w_1(L[j])=\w_2(L[j])$ and $\w_{1|\Pa(L[i])} = \w_{2|\Pa(L[j])}$.

Because functions are right-unique, $\F[L[i]](\w_{1|\Pa(L[i])}) = \F[L[i]](\w_{2|\Pa(L[i])})$. Because $\w(L[i])$ is only modified at iteration $i$, $\w_1(L[i])=\w_2(L[i])$, which contradicts $\w_1(L[i]) \ne \w_2(L[i])$. Therefore,  $\w_1 = \w_2$; the induced complete assignment from an exo-assignment is unique.
\end{itemize}
\end{proof}

Existence entails left-totality; uniqueness entails right-uniqueness, so $\texttt{VFI}(\M, \w_{exo})$ itself is a function of $\w_{exo}$.
Since $\w_{exo}$ functionally determines $\w$ and $\w_{endo} \subseteq \w$, Armstrong's Axioms entail "$\w_{exo}$ functionally determines $\w_{endo}$."

During forward inference, $\G$ is also a computational graph, where edges indicate the order of computation. We start with exo-nodes and the computation "flows down" to endo-nodes, computing their values based on the previously computed values of their parents. Topological sort and graph traversal both take $\Theta(|\V|+|\E|)$ time under adjacency-list representation of graphs. For each $u \in \V_{endo}$, $\F[u]$ is computed exactly once.

\begin{corollary}\label{corollary:permit_exo}
In a finite acyclic SFM $\M$ with nonempty domains, any partial assignment $\w_{|\X}$ over any subset of exo-nodes $\X \subseteq \V_{exo}$ is permitted.
\end{corollary}
\begin{proof}
For every $u \in \V_{exo} \setminus \X$, we assign an arbitrary $\w_{exo}(u) \in \D[u]$ since $\D[u] \ne \emptyset$; for every $u \in \X$, we assign $\w_{exo}(u) \leftarrow \w_{|\X}(u)$, so $\w_{exo} \supseteq \w_{|\X}$. By Theorem \ref{theorem:forward_infer}, $\w=\texttt{VFI}(\M, \w_{exo})$ satisfies $\w_{exo} \subseteq \w \in R_{\M}$, so $\w_{|\X} \subseteq \w_{exo} \subseteq \w \in R_{\M}$ and $\w_{|\X}$ is permitted.
\end{proof}

\begin{corollary}
A finite acyclic SFM $\M$ with nonempty domains is always satisfiable, regardless of its structural functions $\F$.
\end{corollary}
\begin{proof}
Because $\G$ is finite (no infinite regress) and acyclic, there exists at least one root node $u$ (Appendix \ref{trilemma}). Because $\D[u] \ne \emptyset$, we select an arbitrary value $\w_{|\{u\}}(u) \in \D[u]$. Corollary \ref{corollary:permit_exo} says $\w_{|\{u\}}$ is permitted, so $\exists \w \in R_{\M}: \w \supseteq \w_{|\{u\}}$ and $\M$ is satisfiable.
\end{proof}

\subsection{Functional Invariance}\label{infer:invariance}
We use \textit{functional invariance} to describe how a multi-input function's output doesn't change when some inputs have changed: TVs aren't affected by light switches; the output of $f(x, y) = 2x$ is invariant to $y$ given $x$. Notice that \textit{ceteris paribus} (holding other input values constant) is well-defined only if there's a clear input-output distinction given by an underlying function.

In SFM, changing an exo-node's value cannot influence its non-descendants. This is deduced from $\G$ alone. With non-injective functions, new parent values may map to the old child value, resulting in even fewer changed nodes. Equivalently, for $\X \subseteq \V_{exo}$, $\w_{|\V_{exo} \setminus \X}$) functionally determines $\X$'s non-descendants.

\begin{theorem}\label{theorem:invariance}
\textbf{(Invariance in SFM)} In a finite acyclic SFM $\M$ with $\w_1, \w_2 \in R_{\M}$ and changed nodes $\C = \{u \in \V | \w_0(u) \ne \w_1(u)\}$:

If $\w_0(u) \ne \w_1(u)$, then $u \in \bigcup_{v \in \C_{exo}}\mathrm{De}(v)$. ($u$'s value differs in $\w_0$ and $\w_1$ only if it's the descendant of some node in $\C_{exo}$.)
\end{theorem}
\begin{proof}
Let $u \in \C$ be any node such that $\w_0(u) \ne \w_1(u)$. If $u \in \V_{exo}$, then $u \in \C_{exo}$ and we're done. If $u \in \V_{endo}$, then because functions are right-unique, at least one parent $p \in \Pa(u)$ must have a different value ($\w_0(p) \ne \w_1(p)$). We consider $p$ as the new $u$ and repeat this process recursively. Because the SFM graph is finite (no infinite regress) and acyclic, this path $u \leftarrow p_1 \leftarrow p_2 \leftarrow \dots$ must terminate at some exo-node $s \in \V_{exo}$ (Appendix \ref{trilemma}) such that $\w_0(s) \ne \w_1(s)$, which means $s \in \C_{exo}$. The path shows $u$ is a descendant of $s$.
\end{proof}

\subsection{Contrastive Forward Inference}\label{sfm:contrast_forward}
Suppose we already have $\M$ and $\w_0 \in R_{\M}$. To compute $\texttt{VFI}(\M, \w_{1|exo})$, we still need to compute every $\F[u]$. Given all the unchanged nodes from functional invariance, can the graph structure help us reduce $\F[u]$ evaluations?

Yes. With $\C_{exo} = \{u \in \V_{exo}|\w_{0|exo}(u) \ne \w_{1|exo}(u)\}$, the \textit{contrastive forward inference} (CFI) algorithm $\w_1 = \texttt{CFI}(\M, \w_0, \w_{1|\C_{exo}})$ evaluates $\F[u]$ only when at least one parent of $u$ has a changed value, so we don't recompute non-descendants of $\C_{exo}$. \texttt{CFI} evaluates usually fewer (and always no more) structural functions than \texttt{VFI}, especially when $\C$ is small relative to $\V$, when there are many $\w_{i|exo}$ queries relative to one reference $\w_0$, or when many structural functions are non-injective.

\begin{algorithm}[H]
\caption{Contrastive forward inference algorithm}\label{algo:cfi}
\KwData{\begin{itemize}
\item $\M=(\V,\E,\D,\F)$, the SFM.
\item $L = [u_1, u_2, \dots, u_n]$, list of all nodes in topological order (cached).
\item $\w_0$, reference assignment. 
\item $\w_{1|\X}$, exo-assignment over $\X \subseteq \V_{exo}$ for the new forward inference query.
\end{itemize}}

\KwResult{\begin{itemize}
\item $\w_1$, the induced complete assignment from $\w_{1|\C_{exo}} \cup \w_{0|\V_{exo} \setminus \C_{exo}}$.
\end{itemize}}

$\w_1 \gets$  empty dictionary\;
$\mathbf{c} \gets$ empty dictionary; \tcp{whether a node's value changed}
\ForEach{$u \in \V$}{
    \uIf{$u \in \X$ and $\w_{1|\X}[u] \ne \w_{0}[u]$}{
        $\mathbf{c}[u] \gets 1$\\
    }
    \Else{$\mathbf{c}[u] \gets 0$\\}
}
\tcp{traverse the nodes in topological order}
\For{$i \gets 1$ \KwTo $|\V|$}{
$u_i \gets L[i]$\\
$r \gets \mathbf{c}[u_i]$ \tcp{1 if $u_i$ needs re-computation, else 0}
\ForEach{parent $p \in \Pa(u_i)$} {
$r \gets r \lor \mathbf{c}[p]$
}
\uIf{$r = 1$}{
\uIf{$u_i \in V_{exo}$}{
$\w_1[u_i] \gets \w_{1|\mathcal{S}}[u_i]$
}
\Else{
$f \gets \F[u_i]$\\
$\w_{\Pa(u_i)} \gets$ \{$p$: $\w_1(p)$ \textbf{for} $p \in \Pa(u_i)$\}\\
$val \gets f(\w_{\Pa(u_i)})$ \tcp{compute new value for $u_i$}
    \uIf{$val \ne \w_0[u_i]$}{
        $\w_1[u_i] \gets val$\\
        $\mathbf{c}[u_i] \gets 1$
    }
    \Else{
        $\w_1[u_i] \gets \w_0[u_i]$
    }
}
}
\Else{
$\w_1[u_i] \gets \w_0[u_i]$ \tcp{copy from reference assignment}}
}
\end{algorithm}

\clearpage

If we draw an SFM with all arrows pointing downwards, we visually cache a topological order of nodes. We can easily identify the descendants of changed nodes and only evaluate their structural functions, without recomputing the complete assignment.

Unlike functions, contrast isn't a fundamental and irreducible part of causality. It's just a popular heuristic with pragmatic benefits:
\begin{enumerate}
\item Delta compression reduces the length of causal utterances.
\item \texttt{CFI} recomputes (usually) fewer structural functions than \texttt{VFI} during forward inference.
\end{enumerate}

\subsection{Partial Forward Inference}
By modifying depth-first search, we can also design \textit{partial forward inference} algorithms, where we're only interested in a subset of endo-nodes $\Y \subseteq \V_{endo}$, so we don't have to compute values for all endo-nodes. Combined with \texttt{CFI}, it further reduces the number of function evaluations, especially when $\Y$ is much smaller than $\V_{endo}$.

\subsection{Inference in Practice}
\begin{enumerate}
\item \textbf{VFI in Boolean circuits}: A combinational logic circuit \citep{patt2020computing} is a finite acyclic SFM with $\{0, 1\}$ domains and Boolean functions.
Each wire's value is 0 (no electrical current) or 1 (has current). A logic gate receives input wires and returns an output wire, like a structural function.
The output wire of one gate can be the input wire of another gate.
To infer the values of all wires given all input wires, we use VFI and produce causal utterances like "setting this input wire to 1 causes the output wire to be 0."

\item \textbf{CFI in GNU Make}: GNU \texttt{make} is a popular open-source software that automatically determines which pieces of a large program need to be recompiled \citep{gnu_make_2023}. Especially in C and C++, the source code needs to be compiled or linked into a target file, before the target file can be executed by the computer. In a \texttt{Makefile}, there are many rules. Each rule has a target file, a list of source files, and a recipe for compilation. The target file functionally depends on the source files. The target file of one rule can be a source file in another rule. This forms a finite SFM where files are nodes and rules specify edges and structural functions.

VFI compiles all files, but software development is a dynamic process:
We don't compile the files just once. We modify some files, see the results, and repeat.
Because compilation is time-consuming, it's costly to recompile all files after a modification.
Instead, we only need to recompile the descendants of modified files. Just like CFI, \texttt{make} only recompiles the target file if any of its source files (parents) has been modified since the previous compilation, saving lots of time. We can produce causal utterances like "modifying this file causes the final compiled program to crash."
\end{enumerate}

\section{Learning}\label{learning:section}
Learning causal models from statistical data is covered in depth by \citet{pearl2009causality, hernan2020causal, peters2017elements}, so we only discuss some philosophical cases where people prefer some SFM over others, given fully-specified possible worlds and laws-of-nature.

\subsection{Thermometer and Temperature}
We think high room temperature causes high thermometer reading, but not the other way round. Why?

It's common to introduce new nodes and see whether the small model remains true as a sub-SFM of a bigger model. Consider a new node "immersing thermometer in cold water" and all possible worlds are listed below:

\begin{center}
\begin{tabular}{|c|c|c|c|}\hline
Node & \texttt{HighReading} & \texttt{HighTemperature} & \texttt{ColdWater} \\\hline
$\w_1$ & 0 & 0 & 0\\\hline
$\w_2$ & 1 & 1 & 0\\\hline
$\w_3$ & 0 & 0 & 1\\\hline
$\w_4$ & 0 & 1 & 1\\\hline
\end{tabular}
\end{center}

Without granting "intervention" any special status, we see that $\dict{HighReading} \xrightarrow{.} \dict{HighTemperature}$ and $\dict{HighReading, ColdWater} \xrightarrow{.} \dict{HighTemperature}$ aren't true in general, so the edge should point from \texttt{HighTemperature} to \texttt{HighReading}. People prefer simple SFMs that compose well with other SFMs that model the same world.

\subsection{Light, Object, and Shadow}\label{learn:flagpole}
In a symmetric equation involving \dict{Light, Object, Shadow}, any 2 nodes functionally determine the 1 remaining node. Why do we think the shadow is the effect? This asymmetric preference is entailed by people's general dislike of SFM-intersection:
\begin{itemize}
\item With multiple objects, $\dict{Light, Shadow} \xrightarrow{.} \dict{Object}$ isn't true in general. When we add another object whose shadow rests entirely in another object's shadow, the system's light and shadow remain the same, thus violating right-uniqueness.

\item $\dict{Light, Shadow} \xrightarrow{.} \dict{Object}$ cannot SFM-compose with $\dict{Factory} \xrightarrow{.} \dict{Object}$ (objects determined by their production processes). Explicitly encoding both functional dependencies requires SFM-intersection.

\item With one light source and multiple objects, $\dict{Object(i), Shadow(i)} \xrightarrow{.} \dict{Light}$ holds for every \texttt{Object(i)}, resulting in SFM-intersection.

\item $\dict{Object, Shadow} \xrightarrow{.} \dict{Light}$ cannot SFM-compose with $\dict{Hand} \xrightarrow{.} \dict{Light}$ (flashlight direction determined by hand movement) or $\dict{TimeOfDay} \xrightarrow{.} \dict{Light}$ (the Sun's position determined by time of the day), unless we use SFM-intersection.

\item $\dict{Object, Light} \xrightarrow{.} \dict{Shadow}$ can seamlessly compose with upstream and downstream SFMs without SFM-intersection.
\end{itemize}
\section{Benchmark}\label{benchmark:section}
Taking a data-centric approach, we compile a collection of thought experiments about causality and apply SFM to all of them. A good definition of causality should have no trouble fitting these causal scenarios. Unless otherwise mentioned, all domains are binary $\{0, 1\}$.

\subsection{Sensitive to Default}\label{benchmark:base}
\begin{itemize}
\item The assassin shoots the victim, causing the victim's death.
\begin{enumerate}
\item $\texttt{Assassin} \to \texttt{Death}$
\item $\F[\texttt{Death}](\w) = \w(\texttt{Assassin})$
\item Default $\w_c = \dict{Assassin:0, Death:0}$
\item Actual $\w_a = \dict{Assassin:1, Death:1}$
\item $\C_{exo} = \dict{Assassin}, \C_{endo}=\dict{Death}$
\item $\w_{a|\C_{exo}}=\dict{Assassin:1}$ causes $\w_{a|\C_{endo}}=\dict{Death:1}$.
\end{enumerate}

\item At the last moment, the assassin changes his mind and doesn't shoot, causing the victim's survival.
\begin{enumerate}
\item Same $\M$ as above.
\item Default $\w_c = \dict{Assassin:1, Death:1}$
\item Actual $\w_a = \dict{Assassin:0, Death:0}$
\item $\C_{exo} = \dict{Assassin}, \C_{endo}=\dict{Death}$
\item $\w_{a|\C_{exo}}=\dict{Assassin:0}$ causes $\w_{a|\C_{endo}}=\dict{Death:0}$.
\end{enumerate}
\end{itemize}

\subsection{Causal Chain}\label{benchmark:chain}
The assassin shoots a bullet, which kills the victim.
\begin{itemize}
\item The assassin causes both the bullet and the death.

\begin{enumerate}
\item $\texttt{Assassin} \to \texttt{Bullet} \to \texttt{Death}$
\item $\F[\texttt{Bullet}](\w) = \w(\texttt{Assassin})\\ \F[\texttt{Death}](\w) = \w(\texttt{Bullet})$
\item Default $\w_c = \dict{Assassin:0, Bullet:0, Death:0}$
\item Actual $\w_a = \dict{Assassin:1, Bullet:1, Death:1}$
\item $\C_{exo} = \dict{Assassin}, \C_{endo}=\dict{Bullet, Death}$
\item $\w_{a|\C_{exo}}=\dict{Assassin:1}$ causes $\w_{a|\C_{endo}}=\dict{Bullet:1, Death:1}$.
\end{enumerate}

\item (Sub-SFM) The bullet causes the death.

\begin{enumerate}
\item $\texttt{Bullet} \to \texttt{Death}$
\item $\F[\texttt{Death}](\w) = \w(\texttt{Bullet})$
\item Default $\w_c = \dict{Bullet:0, Death:0}$
\item Actual $\w_a = \dict{Bullet:1, Death:1}$
\item $\C_{exo} = \dict{Bullet}, \C_{endo}=\dict{Death}$
\item $\w_{a|\C_{exo}}=\dict{Bullet:1}$ causes $\w_{a|\C_{endo}}=\dict{Death:1}$.
\end{enumerate}
\end{itemize}

\subsection{Connected Double Prevention}
A bodyguard shoots the assassin before the assassin could shoot the victim. The victim survives.
\begin{itemize}
\item The bodyguard causes the assassin's death and the victim's survival.
\begin{enumerate}
\item $\texttt{Bodyguard} \to \texttt{Assassin} \to \texttt{Survive}$
\item $\F[\texttt{Assassin}](\w) = \neg \w(\texttt{Bodyguard})\\ \F[\texttt{Survive}](\w) = \neg\w(\texttt{Assassin})$
\item Actual $\w_a = \dict{Bodyguard:1, Assassin:0, Survive:1}$
\item $\C_{exo} = \dict{Bodyguard}$
\item Tweak $\w_{c|\C_{exo}}=\dict{Bodyguard:0}$
\item Tweaked $\w_c = \texttt{CFI}(\M, \w_a, \w_{c|\C_{exo}}) = \dict{Bodyguard:0, Assassin:1, Survive:0}$
\item $\C_{endo} = \dict{Assassin, Survive}$
\item $\w_{a|\C_{exo}}=\dict{Bodyguard:1}$ causes $\w_{a|\C_{endo}}=\dict{Assassin:0, Survive:1}$.
\end{enumerate}
\end{itemize}

\subsection{Disconnected Double Prevention}
The assassin puts poison in the victim's cup. The bodyguard puts antidote in the cup. The victim survives.
\begin{itemize}
\item Antidote causes the victim's survival.
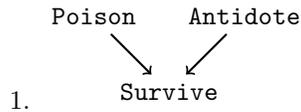
\begin{enumerate}
\item \begin{tikzpicture}
\node(Poison) at (0,1) {\texttt{Poison}};
\node(Antidote) at (2,1) {\texttt{Antidote}};
\node(Survive) at (1,0) {\texttt{Survive}};

\path [->, style=thick] (Poison) edge (Survive);
\path [->, style=thick] (Antidote) edge (Survive);
\end{tikzpicture}
\item $\F[\texttt{Survive}](\w) = \neg \w(\texttt{Poison}) \lor \w(\texttt{Antidote})$
\item Actual $\w_a = \dict{Poison:1, Antidote:1, Survive:1}$
\item $\C_{exo}=\dict{Antidote}$
\item Tweak $\w_{c|\C_{exo}}=\dict{Antidote:0}$
\item Tweaked $\w_c = \texttt{CFI}(\M, \w_a, \w_{c|\C_{exo}}) = \dict{Poison:1, Antidote:0, Survive:0}$
\item $\C_{endo} = \dict{Survive}$
\item $\w_{a|\C_{exo}}=\dict{Antidote:1}$ causes $\w_{a|\C_{endo}}=\dict{Survive:1}$.
\end{enumerate}
\end{itemize}

\subsection{No Appropriate Default}\label{benchmark:no_default}
Two chess players use a coin flip to decide who moves first. If the coin lands on head, the Player 1 moves first; otherwise, Player 2 moves first. It's difficult to identify a "default" world \citep{blanchard2017cause}.

\begin{itemize}
\item Coin landing on head causes Player 1 to move first.
\begin{enumerate}
\item $\texttt{Head} \to \texttt{Player1}$
\item $\F[\texttt{Player1}](\w) = \w(\texttt{Head})$
\item Actual $\w_a = \dict{Head:1, Player1:1}$
\item $\C_{exo}=\dict{Head}$
\item Tweak $\w_{c|\C_{exo}}=\dict{Head:0}$
\item Tweaked $\w_c = \texttt{CFI}(\M, \w_a, \w_{c|\C_{exo}}) = \dict{Head:0, Player1:0}$
\item $\C_{endo} = \dict{Player1}$
\item $\w_{a|\C_{exo}}=\dict{Head:1}$ causes $\w_{a|\C_{endo}}=\dict{Player1:1}$.
\end{enumerate}
\item Coin landing on tail causes Player 2 to move first.
\begin{enumerate}
\item Same $\M$ as above.
\item Actual $\w_a = \dict{Head:0, Player1:0}$
\item $\C_{exo}=\dict{Head}$
\item Tweak $\w_{c|\C_{exo}}=\dict{Head:1}$
\item Tweaked $\w_c = \texttt{CFI}(\M, \w_a, \w_{c|\C_{exo}}) = \dict{Head:1, Player1:1}$
\item $\C_{endo} = \dict{Player1}$
\item $\w_{a|\C_{exo}}=\dict{Head:0}$ causes $\w_{a|\C_{endo}}=\dict{Player1:0}$.
\end{enumerate}
\end{itemize}

\subsection{Gardener and Queen}
The flower lives iff at least one person waters it. The gardener is responsible for watering the flower, but the queen isn't \citep{hart1985causation}. 

\begin{itemize}
\item The gardener's not watering the flower causes the flower's death; the queen's not watering it doesn't cause the flower's death.
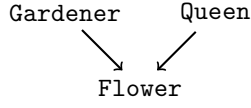
\begin{enumerate}
\item \begin{tikzpicture}
\node(Gardener) at (0,1) {\texttt{Gardener}};
\node(Queen) at (2,1) {\texttt{Queen}};
\node(Flower) at (1,0) {\texttt{Flower}};

\path [->, style=thick] (Gardener) edge (Flower);
\path [->, style=thick] (Queen) edge (Flower);
\end{tikzpicture}
\item $\F[\texttt{Flower}](\w) = \w(\texttt{Gardener}) \lor \w(\texttt{Queen})$
\item Default $\w_c = \dict{Gardener:1, Queen:0, Flower:1}$
\item Actual $\w_a = \dict{Gardener:0, Queen:0, Flower:0}$
\item $\C_{exo} = \dict{Gardener}, \C_{endo}=\dict{Flower}$
\item $\w_{a|\C_{exo}}=\dict{Gardener:0}$ causes $\w_{a|\C_{endo}}=\dict{Flower:0}$.
\end{enumerate}
\end{itemize}

\subsection{OR Firing Squad (Symmetric Overdetermination)}\label{benchmark:overdetermination}
Two assassins simultaneously shoot the victim. It takes only 1 bullet to kill the victim.

\begin{itemize}
\item Both assassins are responsible because "not killing" is default.
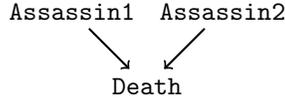
\begin{enumerate}
\item \begin{tikzpicture}
\node(A1) at (0,1) {\texttt{Assassin1}};
\node(A2) at (2,1) {\texttt{Assassin2}};
\node(Death) at (1,0) {\texttt{Death}};

\path [->, style=thick] (A1) edge (Death);
\path [->, style=thick] (A2) edge (Death);
\end{tikzpicture}
\item $\F[\texttt{Death}](\w) = \w(\texttt{Assassin1}) \lor \w(\texttt{Assassin2})$
\item Default $\w_c = \dict{Assassin1:0, Assassin2:0, Death:0}$
\item Actual $\w_a = \dict{Assassin1:1, Assassin2:1, Death:1}$
\item $\C_{exo} = \dict{Assassin1, Assassin2}, \C_{endo}=\dict{Death}$
\item $\w_{a|\C_{exo}}=\dict{Assassin1:1, Assassin2:1}$ causes $\w_{a|\C_{endo}}=\dict{Death:1}$.
\end{enumerate}

\item Assassin 1 causes nothing because had he not shot, Assassin 2 would've still killed the victim.
\begin{enumerate}
\item Same $\M$ as above.
\item Actual $\w_a = \dict{Assassin1:1, Assassin2:1, Death:1}$
\item $\C_{exo}=\dict{Assassin1}$
\item Tweak $\w_{c|\C_{exo}}=\dict{Assassin1:0}$
\item Tweaked $\w_c = \texttt{CFI}(\M, \w_a, \w_{c|\C_{exo}}) = \dict{Assassin1:0, Assassin2:1, Death:1}$
\item $\C_{endo} = \emptyset$
\item $\w_{a|\C_{exo}}=\dict{Assassin1:1}$ causes $\w_{a|\C_{endo}}=\emptyset$.
\end{enumerate}
\end{itemize}

\subsection{AND Firing Squad}\label{benchmark:and_firing}
2 assassins simultaneously shoot the victim. It takes at least 2 bullets to kill the victim.

\begin{itemize}
\item Both assassins are responsible because "not killing" is default.
\begin{enumerate}
\item \begin{tikzpicture}
\node(A1) at (0,1) {\texttt{Assassin1}};
\node(A2) at (2,1) {\texttt{Assassin2}};
\node(Death) at (1,0) {\texttt{Death}};

\path [->, style=thick] (A1) edge (Death);
\path [->, style=thick] (A2) edge (Death);
\end{tikzpicture}
\item $\F[\texttt{Death}](\w) = \w(\texttt{Assassin1}) \land \w(\texttt{Assassin2})$
\item Default $\w_c = \dict{Assassin1:0, Assassin2:0, Death:0}$
\item Actual $\w_a = \dict{Assassin1:1, Assassin2:1, Death:1}$
\item $\C_{exo} = \dict{Assassin1, Assassin2}, \C_{endo}=\dict{Death}$
\item $\w_{a|\C_{exo}}=\dict{Assassin1:1, Assassin2:1}$ causes $\w_{a|\C_{endo}}=\dict{Death:1}$.
\end{enumerate}

\item Assassin 1 is individually responsible because had he not shot, the victim would've survived.
\begin{enumerate}
\item Same $\M$ as above.
\item Actual $\w_a = \dict{Assassin1:1, Assassin2:1, Death:1}$
\item $\C_{exo}=\dict{Assassin1}$
\item Tweak $\w_{c|\C_{exo}}=\dict{Assassin1:0}$
\item Tweaked $\w_c = \texttt{CFI}(\M, \w_a, \w_{c|\C_{exo}}) = \dict{Assassin1:0, Assassin2:1, Death:0}$
\item $\C_{endo} = \dict{Death}$
\item $\w_{a|\C_{exo}}=\dict{Assassin1:1}$ causes $\w_{a|\C_{endo}}=\dict{Death:1}$.
\end{enumerate}
\end{itemize}

\subsection{Connected Preemption}\label{benchmark:preemption_1}
Assassin 1 shoots the victim first. If the victim doesn't die, Assassin 2 will shoot. Had Assassin 1 not shot, the victim still would've died.

\begin{itemize}
\item Assassin 1 causes the victim's death and Assassin 2's not-shooting.
\begin{enumerate}
\item \begin{tikzpicture}
\node(A1) at (0,2) {\texttt{Assassin1}};
\node(A2) at (4,2) {\texttt{Assassin2}};
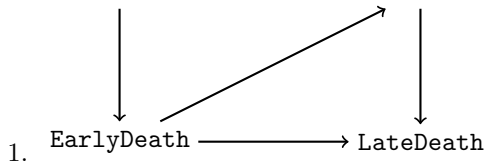
\node(D1) at (0,0) {\texttt{EarlyDeath}};
\node(D2) at (4,0) {\texttt{LateDeath}};

\path [->, style=thick] (A1) edge (D1);
\path [->, style=thick] (D1) edge (A2);
\path [->, style=thick] (D1) edge (D2);
\path [->, style=thick] (A2) edge (D2);
\end{tikzpicture}
\item $\F[\texttt{EarlyDeath}](\w) = \w(\texttt{Assassin1})\\\F[\texttt{Assassin2}](\w) = \neg \w(\texttt{EarlyDeath})\\\F[\texttt{LateDeath}](\w) = \w(\texttt{EarlyDeath}) \lor \w(\texttt{Assassin2})$
\item Actual $\w_a = \dict{Assassin1:1, EarlyDeath:1, Assassin2:0, LateDeath:1}$
\item $\C_{exo}=\dict{Assassin1}$
\item Tweak $\w_{c|\C_{exo}}=\dict{Assassin1:0}$
\item Tweaked $\w_c = \texttt{CFI}(\M, \w_a, \w_{c|\C_{exo}})\\= \dict{Assassin1:0, EarlyDeath:0, Assassin2:1, LateDeath:1}$
\item $\C_{endo} = \dict{EarlyDeath, Assassin2}$
\item $\w_{a|\C_{exo}}=\dict{Assassin1:1}$ causes $\w_{a|\C_{endo}}=\dict{EarlyDeath:1, Assassin2:0}$.
\item We cannot say $\dict{Assassin1:1}$ causes $\dict{LateDeath:1}$ because $\texttt{Latedeath} \notin \C_{endo}$.
\end{enumerate}
\end{itemize}

\subsection{Disconnected Preemption}\label{benchmark:preemption_2}
Assassin 1 shoots the victim first. Several moments later, Assassin 2 shoots unconditionally.

\begin{itemize}
\item Assassin 1 causes the victim's death.
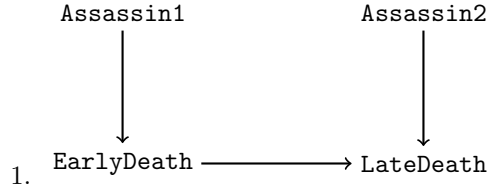
\begin{enumerate}
\item \begin{tikzpicture}
\node(A1) at (0,2) {\texttt{Assassin1}};
\node(A2) at (4,2) {\texttt{Assassin2}};
\node(D1) at (0,0) {\texttt{EarlyDeath}};
\node(D2) at (4,0) {\texttt{LateDeath}};

\path [->, style=thick] (A1) edge (D1);
\path [->, style=thick] (D1) edge (D2);
\path [->, style=thick] (A2) edge (D2);
\end{tikzpicture}
\item $\F[\texttt{EarlyDeath}](\w) = \w(\texttt{Assassin1})\\\F[\texttt{LateDeath}](\w) = \w(\texttt{EarlyDeath}) \lor \w(\texttt{Assassin2})$
\item Actual $\w_a = \dict{Assassin1:1, Assassin2:1, EarlyDeath:1, LateDeath:1}$
\item $\C_{exo}=\dict{Assassin1}$
\item Tweak $\w_{c|\C_{exo}}=\dict{Assassin1:0}$
\item Tweaked $\w_c = \texttt{CFI}(\M, \w_a, \w_{c|\C_{exo}})\\= \dict{Assassin1:0, Assassin2:1, EarlyDeath:0, LateDeath:1}$
\item $\C_{endo} = \dict{EarlyDeath}$
\item $\w_{a|\C_{exo}}=\dict{Assassin1:1}$ causes $\w_{a|\C_{endo}}=\dict{EarlyDeath:1}$.
\end{enumerate}
\end{itemize}

Time difference distinguishes preemption from symmetric overdetermination: To an extreme, we wouldn't regard immediate death and death in 100 years as the same event.

\subsection{Relevant Background Conditions}\label{benchmark:background}
Ignition requires both striking the match and oxygen present, but we only mention striking the match as the cause of fire.

\begin{itemize}
\item Striking the match causes ignition.
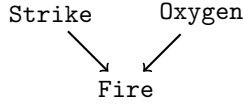
\begin{enumerate}
\item \begin{tikzpicture}
\node(Strike) at (0,1) {\texttt{Strike}};
\node(Oxygen) at (2,1) {\texttt{Oxygen}};
\node(Fire) at (1,0) {\texttt{Fire}};

\path [->, style=thick] (Strike) edge (Fire);
\path [->, style=thick] (Oxygen) edge (Fire);
\end{tikzpicture}
\item $\F[\texttt{Fire}](\w) = \w(\texttt{Strike}) \land \w(\texttt{Oxygen})$
\item Default $\w_c = \dict{Strike:0, Oxygen:1, Fire:0}$
\item Actual $\w_a = \dict{Strike:1, Oxygen:1, Fire:1}$
\item $\C_{exo} = \dict{Strike}, \C_{endo}=\dict{Fire}$
\item $\w_{a|\C_{exo}}=\dict{Strike:1}$ causes $\w_{a|\C_{endo}}=\dict{Fire:1}$.
\end{enumerate}

\item While repeatedly striking a match in an oxygen-deprived container, there's no ignition. Pumping in oxygen causes the match to ignite.
\begin{enumerate}
\item Same $\M$ as above.
\item Default $\w_c = \dict{Strike:1, Oxygen:0, Fire:0}$
\item Actual $\w_a = \dict{Strike:1, Oxygen:1, Fire:1}$
\item $\C_{exo} = \dict{Oxygen}, \C_{endo}=\dict{Fire}$
\item $\w_{a|\C_{exo}}=\dict{Oxygen:1}$ causes $\w_{a|\C_{endo}}=\dict{Fire:1}$.
\end{enumerate}
\end{itemize}

Similarly, a criminal wouldn't have committed the crime had the universe not existed/had he never been born, but we don't consider those as causes of the crime.

\subsection{Irrelevant Background Conditions}\label{benchmark:irrelevant}
The assassin simultaneously shoots the victim and whispers.

\begin{itemize}
\item Whispering doesn't cause anything.
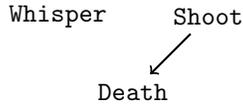
\begin{enumerate}
\item \begin{tikzpicture}
\node(Whisper) at (0,1) {\texttt{Whisper}};
\node(Shoot) at (2,1) {\texttt{Shoot}};
\node(Death) at (1,0) {\texttt{Death}};

\path [->, style=thick] (Shoot) edge (Death);
\end{tikzpicture}
\item $\F[\texttt{Death}](\w) = \w(\texttt{Shoot})$
\item Actual $\w_a = \dict{Whisper:1, Shoot:1, Death:1}$
\item $\C_{exo}=\dict{Whisper}$
\item Tweak $\w_{c|\C_{exo}}=\dict{Whisper:0}$
\item Tweaked $\w_c = \texttt{CFI}(\M, \w_a, \w_{c|\C_{exo}}) = \dict{Whisper:0, Shoot:1, Death:1}$
\item $\C_{endo} = \emptyset$
\item $\w_{a|\C_{exo}}=\dict{Whisper:1}$ causes $\w_{a|\C_{endo}}=\emptyset$.
\end{enumerate}

\item Shooting causes death.
\begin{enumerate}
\item Same $\M, \w_a$ as above.
\item $\C_{exo}=\dict{Shoot}$
\item Tweak $\w_{c|\C_{exo}}=\dict{Shoot:0}$
\item Tweaked $\w_c = \texttt{CFI}(\M, \w_a, \w_{c|\C_{exo}}) = \dict{Whisper:1, Shoot:0, Death:0}$
\item $\C_{endo} = \dict{Death}$
\item $\w_{a|\C_{exo}}=\dict{Shoot:1}$ causes $\w_{a|\C_{endo}}=\dict{Death:1}$.
\end{enumerate}

\end{itemize}

Similarly, Socrates drinks hemlock at dusk and dies. Hemlock causes death, but dusk doesn't cause anything \citep{achinstein1975causation}.

\subsection{Boulder and Hiker}\label{benchmark:boulder}
A hiker sees a boulder rolling towards him, so he dodges and survives. Had he not dodged, he wouldn't have survived \citep{hitchcock2001intransitivity}. This is an ostensible counterexample to the transitivity of causation (boulder causes dodge, dodge causes survival, but boulder doesn't cause survival). "Transitivity" is better understood as SFM-composition.

\begin{itemize}
\item Boulder causes dodge and doesn't cause survival.
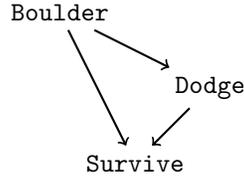
\begin{enumerate}
\item \begin{tikzpicture}
\node(Boulder) at (0,2) {\texttt{Boulder}};
\node(Dodge) at (2,1) {\texttt{Dodge}};
\node(Survive) at (1,0) {\texttt{Survive}};

\path [->, style=thick] (Boulder) edge (Dodge);
\path [->, style=thick] (Boulder) edge (Survive);
\path [->, style=thick] (Dodge) edge (Survive);
\end{tikzpicture}
\item $\F[\texttt{Dodge}](\w) = \w(\texttt{Boulder})\\\F[\texttt{Survive}](\w)=\neg \w(\texttt{Boulder}) \lor \w(\texttt{Dodge})$
\item Actual $\w_a = \dict{Boulder:1, Dodge:1, Survive:1}$
\item $\C_{exo}=\dict{Boulder}$
\item Tweak $\w_{c|\C_{exo}}=\dict{Boulder:0}$
\item Tweaked $\w_c = \texttt{CFI}(\M, \w_a, \w_{c|\C_{exo}}) = \dict{Boulder:0, Dodge:0, Survive:1}$
\item $\C_{endo} = \dict{Dodge}$
\item $\w_{a|\C_{exo}}=\dict{Boulder:1}$ causes $\w_{a|\C_{endo}}=\dict{Dodge:1}$, but $\texttt{Survive} \notin \C_{endo}$.
\end{enumerate}

\item (Sub-SFM) Dodge causes survival.
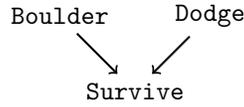
\begin{enumerate}
\item \begin{tikzpicture}
\node(Boulder) at (0,1) {\texttt{Boulder}};
\node(Dodge) at (2,1) {\texttt{Dodge}};
\node(Survive) at (1,0) {\texttt{Survive}};

\path [->, style=thick] (Boulder) edge (Survive);
\path [->, style=thick] (Dodge) edge (Survive);
\end{tikzpicture}
\item $\F[\texttt{Survive}](\w)=\neg \w(\texttt{Boulder}) \lor \w(\texttt{Dodge})$
\item Actual $\w_a = \dict{Boulder:1, Dodge:1, Survive:1}$
\item $\C_{exo}=\dict{Dodge}$
\item Tweak $\w_{c|\C_{exo}}=\dict{Dodge:0}$
\item Tweaked $\w_c = \texttt{CFI}(\M, \w_a, \w_{c|\C_{exo}}) = \dict{Boulder:1, Dodge:0, Survive:0}$
\item $\C_{endo} = \dict{Survive}$
\item $\w_{a|\C_{exo}}=\dict{Dodge:1}$ causes $\w_{a|\C_{endo}}=\dict{Survive:1}$.
\end{enumerate}

\end{itemize}

\subsection{Bogus Prevention}\label{benchmark:bogus_prevent}
Taking birth control pills is the cause of a woman not getting pregnant, but not the cause of a man not getting pregnant, although "birth control prevents pregnancy" is always true \citep{salmon1971statistical}.

\begin{itemize}
\item Birth control causes a woman to be unable to get pregnant.
\begin{enumerate}
\item \begin{tikzpicture}
\node(IsWoman) at (0,1) {\texttt{IsWoman}};
\node(BirthControl) at (2,1) {\texttt{BirthControl}};
\node(CanPregnant) at (1,0) {\texttt{CanPregnant}};

\path [->, style=thick] (IsWoman) edge (CanPregnant);
\path [->, style=thick] (BirthControl) edge (CanPregnant);
\end{tikzpicture}
\item $\F[\texttt{CanPregnant}](\w)= \w(\texttt{IsWoman}) \land \neg \w(\texttt{BirthControl})$
\item Actual $\w_a = \dict{IsWoman:1, BirthControl:1, CanPregnant:0}$
\item $\C_{exo}=\dict{BirthControl}$
\item Tweak $\w_{c|\C_{exo}}=\dict{BirthControl:0}$
\item Tweaked $\w_c = \texttt{CFI}(\M, \w_a, \w_{c|\C_{exo}}) = \dict{IsWoman:1, BirthControl:0, CanPregnant:1}$
\item $\C_{endo} = \dict{CanPregnant}$
\item $\w_{a|\C_{exo}}=\dict{BirthControl:1}$ causes $\w_{a|\C_{endo}}=\dict{CanPregnant:0}$.
\end{enumerate}

\item Birth control doesn't cause anything for a man.
\begin{enumerate}
\item Same $\M$ as above.
\item Actual $\w_a = \dict{IsWoman:0, BirthControl:1, CanPregnant:0}$
\item $\C_{exo}=\dict{BirthControl}$
\item Tweak $\w_{c|\C_{exo}}=\dict{BirthControl:0}$
\item Tweaked $\w_c = \texttt{CFI}(\M, \w_a, \w_{c|\C_{exo}}) = \dict{IsWoman:0, BirthControl:0, CanPregnant:0}$
\item $\C_{endo} = \emptyset$
\item $\w_{a|\C_{exo}}=\dict{BirthControl:1}$ causes $\w_{a|\C_{endo}}=\emptyset$.
\end{enumerate}
\end{itemize}

\subsection{Backtracking Counterfactuals}
Subjunctive conditionals \citep{sep-counterfactuals} use forward inference, while indicative/backtracking/non-causal conditionals don't.

\begin{itemize}
\item (Subjunctive) If Shakespeare didn't write Hamlet, someone else would have.
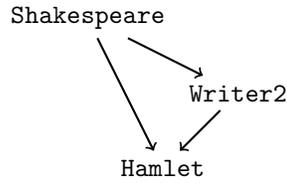
\begin{enumerate}
\item \begin{tikzpicture}
\node(Boulder) at (0,2) {\texttt{Shakespeare}};
\node(Dodge) at (2,1) {\texttt{Writer2}};
\node(Survive) at (1,0) {\texttt{Hamlet}};

\path [->, style=thick] (Boulder) edge (Dodge);
\path [->, style=thick] (Boulder) edge (Survive);
\path [->, style=thick] (Dodge) edge (Survive);
\end{tikzpicture}
\item $\F[\texttt{Writer2}](\w) = \neg \w(\texttt{Shakespeare})\\\F[\texttt{Hamlet}](\w)=\w(\texttt{Shakespeare}) \lor \w(\texttt{Writer2})$
\item Actual $\w_{a}=\dict{Shakespeare:1, Writer2:0, Hamlet:1}$.
\item Query $\w_{exo}=\dict{Shakespeare:0}$
\item Queried $\w=\texttt{VFI}(\M, \w_{exo}) = \dict{Shakespeare:0, Writer2:1, Hamlet:1}$.
\end{enumerate}

\item (Indicative) If Shakespeare didn't write Hamlet, someone else did.

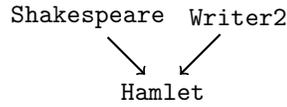
\begin{enumerate}
\item \begin{tikzpicture}
\node(Boulder) at (0,1) {\texttt{Shakespeare}};
\node(Dodge) at (2,1) {\texttt{Writer2}};
\node(Survive) at (1,0) {\texttt{Hamlet}};

\path [->, style=thick] (Boulder) edge (Survive);
\path [->, style=thick] (Dodge) edge (Survive);
\end{tikzpicture}
\item $\F[\texttt{Hamlet}](\w)= \w(\texttt{Shakespeare}) \lor \w(\texttt{Writer2})$
\item Given $\M$, the only $\w_{|\dict{Writer2}}$ compatible with $\dict{Shakespeare:0, Hamlet:1}$ is $\dict{Writer2:1}$.
\end{enumerate}
\end{itemize}

\subsection{Impossible Interventions}
Unlike $\texttt{Assassin}\to\texttt{Death}$, some functional dependencies contain all parents due to how the child is logically/conceptually/metaphysically defined, so it's impossible to add surgical interventions:
\begin{enumerate}
\item The string "hello" is functionally determined by its first character being "h", second character being "e", \dots 
\item The average height of students in the class is functionally determined by the individual height of each student.
\item Winning 2 out of 3 rounds is functionally determined by the result of each round.
\end{enumerate}
They're often known as supervenience (Section \ref{apply:supervenience}).

\section{Discussion}\label{discuss:section}
\subsection{Is SFM Insufficient?}
Some may argue that since SFM and functions can have non-causal interpretations, they are insufficient for defining causality. We respond with 3 counterarguments:
\begin{enumerate}
\item Some examples of insufficiency are results of misinterpretation. For example, student ID functionally determines all attributes (name, age, course registration, etc.) of a student in a database, but changing a student's ID won't cause changes in those attributes. This example doesn't hold because if we allow arbitrary changes to ID, there could be repeated IDs in different rows and ID no longer functionally determines other attributes.

\item Incorrect causal models (e.g. "cancer causes smoking") are still causal, unlike non-causal models (e.g. correlations, symmetric equations), which don't use functions at all. Since SFM comes from the conceptual analysis of what "causation" should mean, its definition cannot include all empirical facts about our world.

\item People often use causal interpretations to understand purely mathematical functions. When we say "changing the independent variable $x$ causes the dependent variable $y$ to change," we're using CFI.
\end{enumerate}

\subsection{A Case Against Actual Causality}\label{discuss:actual}
Delta compression and CFI are slightly useful heuristics that also fit our intuitions. However, the assumption that there exists a fixed set of "actual causes" is questionable in complex systems.

\begin{example}
A circuit has $n$ binary switches $\V_{exo}=\{X_1, X_2, \dots, X_n\}$ and 1 light bulb $\V_{endo}=\{Y\}$, where the $n$ switches functionally determine the light via a Boolean function $f: \{0, 1\}^n \to \{0, 1\}$.

Given the state of all switches and the light, which switches are the "actual causes" of the light being on/off?

There are $2^{2^n}$ different $n$-input $1$-output Boolean functions $f$. For each $f$, there are $2^n$ different possible worlds. Proponents of actual causality must accept one of the following:
\begin{enumerate}
\item Provide an algorithm that can identify actual causes in $2^{2^n}\times 2^n = 2^{(2^n+n)}$ situations. Case-by-case analyses won't scale.
\item Admit that contrast, default, and actual causality belong to an imperfect mental heuristic that would fail in complex systems.
\end{enumerate}

Graphical models don't help because there's only one Boolean function and we shouldn't insert hypothetical intermediate nodes. For $\X \subseteq \V_{exo}$, SFM can answer "is $\w_{0|\X}$ the cause of $\w_{0|endo}$?" by tweaking $\w_{0|\X}$ into $\w_{1|\X}$, inferring $\w_1 = \texttt{CFI}(\M, \w_0, \w_{1|\X})$, and contrasting $\w_1(Y)$ against $\w_0(Y)$ . But "a fixed set of actual causes" given $\M$ and $\w_0$ remains ill-defined.
People's intuitions may not give consistent answers and even if they do, such answers provide less information about $f$ than the input-output mappings of $f$ itself.
\end{example}

This example generalizes all "difficult causal scenarios" with binary variables and following features:
\begin{enumerate}
\item \textbf{Causal}: We can manipulate the switches to control the light.
\item \textbf{Deterministic}: It has no probabilistic component.
\item \textbf{Fully-specified}: Epistemological skepticism like "how do we know these laws-of-nature are true" doesn't apply.
\item \textbf{Clear input-output distinction}: There's no ambiguity in the direction of causal arrows.
\end{enumerate}
Therefore, proposing and solving a few cases wouldn't dissolve our objection.

Intuitions are often unreliable for modeling reality. Outside simple, everyday causal utterances, there's no real downside in abandoning actual causality. $\M$ itself perfectly describes the causal system and answers all "what if" inference queries like $\texttt{VFI}(\M, \w_{exo})$. Instead of listing "actual causes," a scientist should try modeling the functional determinations in a system.

Actual causality is almost only used in normative theories (e.g. responsibility, blame, proximate causes, ethics, law \citep{sep-causation-law}), which handle disagreements when everyone agrees on $\M$ (laws-of-nature) and $\w_a$ (what actually happens). Working with full SFMs instead of actual causality allows us to consider strictly more normative theories.

\section{Probabilistic SFM}\label{prob:section}
To incorporate probability theory, we don't need to modify the definition of SFM. We just extend domains $\D[u]$ to include random variables and modify structural functions $\F[u]$ accordingly. Probability isn't required for most thought experiments on causality, but we'll provide a rigorous mathematical foundation for probabilistic SFM. Notably, \textit{nodes and random variables are not the same}. We avoid calling nodes "variables" precisely for this reason.

\subsection{Probabilistic Extension}
Think of a node $u$ as a name or index. Its value $\w(u)$ can be a random variable: $\w(u) = X$. A \textit{random variable} $X: \Omega \to \mathbb{R}$ maps an outcome $\omega$ (in sample space $\Omega$) to a real number $X(\omega) \in \mathbb{R}$. "$X=x$" is a shorthand for event $\{\omega \in \Omega | X(\omega) = x\}$, so we can compute its probability $\Pr[X=x]$. $X=x$ isn't an actual equation because $X$ is a function and $x$ is a real number. Again, "node $u$ has value $X$; $X$ is a random variable" and "random variable $X$ takes on value $x$; $x$ is a real number" are different things.

Most basically, \textit{functions of random variables} are actually function compositions \citep{blitzstein2015introduction}. Consider real-valued function $f(x)=2x$ and random variable $X: \Omega \to \mathbb{R}$. We want a new random variable $Y$ that always "takes twice the value" of $X$:
\begin{align*}
Y(\omega) &= 2X(\omega)\\
&= f(X(\omega))\\
&= (f \circ X)(\omega)\\
Y &= f \circ X
\end{align*}
The expression "$Y = f(X)$" is wrong by a rigorous standard, because random variable $X$ isn't in $f$'s domain of real numbers.

Formally, the \textit{probabilistic extension} of SFM $\M_{old} = (\V, \E, \D_{old}, \F_{old})$ returns a new SFM $\M_{new}=(\V, \E, \D_{new}, \F_{new})$:
\begin{enumerate}
\item Specify the sample space $\Omega$.
\item Specify a set of nodes $\mathcal{S} \subseteq \V$ for probabilistic extension. Other nodes $\V \setminus \mathcal{S}$ still don't have random variables in their domains.

The set $\mathcal{S}$ must be downward closed: if $u \in \mathcal{S}$, then every descendant of $u$ must also be in $\mathcal{S}$, as if randomness is "contagious" and flows down the computational graph.

For random variables to be well-defined, we also require $\D_{old}[u]$ (e.g. real numbers, vectors, graphs, functions) to be measurable for all $u \in \mathcal{S}$.
\item Let $RV[u]$ denote the set of random variables $\Omega \to \D_{old}[u]$. If $u \in \mathcal{S}$, $\D_{new}[u] = \D_{old}[u] \cup RV[u]$; otherwise, $\D_{new}[u] = \D_{old}[u]$.
\item Recall that a random variable $X$ has realization $X(\omega)$ given outcome $\omega \in \Omega$.

For nodes $\X \subseteq \V$, we define the \textit{realization} of assignment $\w_{|\X}$ given outcome $\omega$ as:
$$\texttt{Realize}(\w_{|\X}, \omega) = \{u : \w_{|\X}(u) \textbf{ if } \w_{|\X}(u) \in \D_{old}[u] \textbf{ else } (\w_{|\X}(u))(\omega)\}_{u \in \X}$$
By realizing every random variable with $\omega$ and keeping other values as is, $\texttt{Realize}(\w_{|\X}, \omega)$ is an assignment of both $\M_{old}$ and $\M_{new}$ because $\forall u \in \V: \texttt{Realize}(\w_{|\X}, \omega)(u) \in \D_{old}[u]$.

\item For $\F_{new}$ and endo-node $u \in \V_{endo}$:
\begin{itemize}
\item If $\forall p \in \Pa(u): \w_{|\Pa(u)}(p) \in \D_{old}[p]$ (no parent value is a random variable), then $\F_{new}[u](\w_{|\Pa(u)}) = \F_{old}[u](\w_{|\Pa(u)}) \in \D_{old}[u]$.
\item Otherwise (at least one parent value is a random variable), $\F_{new}[u](\w_{|\Pa(u)})$ is a random variable $\Omega \to \D_{old}[u]$ in $RV[u]$. For outcome $\omega \in \Omega$, we compute $\F_{new}[u](\w_{|\Pa(u)})(\omega) = \F_{old}[u](\texttt{Realize}(\w_{|\Pa(u)}, \omega))$.
\end{itemize}
\end{enumerate}

Some corollaries about probabilistic extension:
\begin{enumerate}
\item For every $u \in \V$, $\F_{new}[u] \supseteq \F_{old}[u]$ because both the domain and the codomain are strictly extended, hence the name "probabilistic extension."
\item If $\w$ satisfies $\mathcal{M}_{old}$, then it also satisfies $\mathcal{M}_{new}$. For $\X \subseteq \V$, if $\w_{|\X}$ is permitted by $\mathcal{M}_{old}$, then it's also permitted by $\mathcal{M}_{new}$.
\item If $\w$ satisfies $\M_{new}$, then its realization $\texttt{Realize}(\w, \omega)$ also satisfies $\M_{old}$ for every $\omega \in \Omega$.

Intuitively, random variables express uncertainty about which realization is actual. Each realization is a possible world in $\M_{old}$. Probability merely adds "weights" to these possible worlds, so causal mechanisms are deterministic and true in every realization. This is unlike Bayesian networks, where the mechanisms are inherently random.
\end{enumerate}

We can now formalize "correlation doesn't imply causation" using SFM: The same "observational distribution" $\w_{|\X}$ (where some nodes have random variables as values; $\X \subseteq \V$) might be permitted by different SFMs $\M_1 \ne \M_2$ with $R_{\M_1} \ne R_{\M_2}$, which cannot be treated as equal.

\subsection{Bayesian Networks}
With probabilistic extension, SFM generalizes Bayesian networks, which also use directed acyclic graphs. In a Bayesian network \citep{koller2009pgm}, each node corresponds to a random variable, each exo-node stores a marginal distribution, and each endo-node stores a conditional distribution given the node's parents.

Bayesian networks require the exogenous random variables to be probabilistically independent, while we don't enforce that requirement (you may enforce it explicitly).

It's difficult for Bayesian networks to represent SFM. To encode functional determination (right-uniqueness), the conditional distributions must be degenerate. When input distributions cannot be assumed (e.g. light switch doesn't affect TV) and we only have specific input values, the marginal distributions are degenerate too. This sacrifices nearly all expressiveness of a Bayesian network.

Any Bayesian network can be expressed by an SFM. We'll use \textit{probability integral transform} (PIT) to represent conditional probability distributions with deterministic functions:
\begin{enumerate}
\item For simplicity, consider real-valued ($\Omega \to \mathbb{R}$) random variables $Y, X_1, X_2, \dots, X_n$ and conditional distribution $\Pr[Y|X_1=x_1, X_2=x_2, \dots, X_n=x_n]$.
\item Create continuous uniform random variable $U \sim \mathrm{Unif}(0, 1)$ in range $[0, 1]$, independent from all $X_i$.
\item Let $F_{Y|X_i=x_i}(y): \mathbb{R} \to [0, 1]$ be the conditional CDF (cumulative distribution function) of $Y$, such that it has an inverse $F_{Y|X_i=x_i}^{-1}: [0, 1] \to \mathbb{R}$.
\item By PIT \citep{blitzstein2015introduction}, random variable $F_{Y|X_i=x_i}^{-1} \circ U$ has exactly the same CDF as $F_{Y|X_i=x_i}$.
\item We've created a deterministic function $f(x_1, x_2, \dots, x_n, U)$ that returns a random variable $F_{Y|X_i=x_i}^{-1} \circ U$, given real-valued $x_i$ and random variable $U$.
\end{enumerate}

Essentially, we can enforce "all mechanisms are deterministic" without sacrificing expressiveness. The inherent randomness of a mechanism is "injected" by an unobservable "noise" parent whose value is a random variable. This practice is fairly common:
\begin{enumerate}
\item To sample values from $\mathcal{N}(\mu, \sigma^2)$, the \textit{reparameterization trick} uses deterministic function $f(\mu, \sigma, \epsilon)=\mu + \sigma \times \epsilon$, where $\epsilon$ is sampled from an auxiliary "noise" distribution $\mathcal{N}(0, 1)$ \citep{kingma2013vae}.
\item \textit{Additive noise model} $Y = f_Y(X) + N_Y$ has random variables $X, Y, N_Y$, deterministic function $f_Y$, and additive noise $N_Y \perp X$ \citep{peters2017elements}.
\item Randomness in computer programs often comes from built-in random number generators, while the main program is deterministic.
\end{enumerate}

\section{Comparison}\label{compare}
\subsection{Symmetric Laws and Causal Eliminativism}
In a \textit{symmetric equation} of $n$ variables, the values of any $n-1$ variables functionally determine the value of the 1 remaining variable. Newton's second law of motion $F=ma$, the ideal gas law $pV=nRT$, and Ohm's law $V=IR$ are symmetric laws. This differs from the non-injective asymmetry of functions.
We usually view symmetric equations as non-causal, because 1 equation is simpler than $n$ functional determinations.

As a causal eliminativist, \citet{russell1912notion} argues that causality doesn't appear in physics and should be removed from philosophy altogether. However, we've shown that functions and SFM are useful.
We only consider Russell's attack on the functional theory of causation, since we don't agree with other definitions either.
\begin{enumerate}
\item Plurality of causes: Multiple alternative causes like gunshot, arsenic, etc. can map to the same effect - the person's death. (Some functions are non-injective.)
\item Plurality of effects: The effect can be defined as the whole state of the world, which contains many variables. (The "cause" node has multiple descendants.)
\end{enumerate}
Russell incorrectly dismisses functional asymmetry (non-injectiveness) as "illusory," as if the plurality of effects makes both sides symmetric. But these "pluralities" aren't the same. Non-injectiveness cannot be eliminated without changing the function itself.

SFM also addresses other eliminativist challenges on causality \citep{sep-causation-physics}. SFM-causality isn't vague; actual causality, while not appearing in physics, is a slightly useful heuristic that can be abandoned when necessary; probabilistic extension handles inherently random mechanisms; functions are compatible with different theories of space (e.g. action at a distance) and time (Section \ref{apply:time}).

\subsection{Hume, Regularity, and Problem of Induction}\label{compare:hume}
\citet{hume1896treatise} challenges causality as follows. We say "striking a match causes it to ignite." But empirically, we only observe constant conjunctions of events like "match struck" followed by "match igniting." We don't directly observe the link/connection between cause and effect. So any causal "law" is an inductive generalization from particular events, with no necessary guarantee to remain true in the future \citep{sep-hume}.

Hume conflates 2 distinct problems:
\begin{enumerate}
\item \textbf{Conceptual}: What's the definition of causality?

Claiming "causality is just a special kind of regularity" is true but non-reductive: What is that "special kind"? All inductive models (e.g. correlations, symmetric equations) model "regular connections," but only functional determination captures our causal intuition.

Besides relying on unspecified physical/metaphysical models (e.g. time, space, contiguity), regularity conditions like "all events of type $X$ are followed by an event of type $Y$" \citep{sep-causation-regularity} cannot produce causal utterances in background condition cases (Section \ref{benchmark:background}), which are deterministic and fully-specified.

\item \textbf{Epistemological}: How to ensure the correctness of a causal model?

Non-probabilistic SFM (due to right-uniqueness) and symmetric laws make exceptionless claims about reality, while correlation doesn't. Perhaps that's why Hume attacks causality first. However, all inductive generalizations from empirical data are equally susceptible to the Problem of Induction (PoI) \citep{sep-induction-problem}. Causality isn't somehow "more unreliable" than symmetric laws or correlations.

We formulate PoI as follows. Consider a normal world $W_1(t)$ and a piecewise world $W_2(t)$. $W_2(t)$ is exactly the same as $W_1(t)$ for all time $t$ before $t_0$, but is drastically different after $t_0$. Given a world $W$ and all its information before $t_0$, there's no way of distinguishing whether $W$ is $W_1$ or $W_2$.

By enumerating different ways of $W_2$ being "drastically different", such as "the world exploding after $t_0$" or "the gravitational constant doubling after $t_0$", we can construct worlds where symmetric laws and correlations break down under PoI. Therefore, PoI isn't an attack against causality alone. Similarly, a conceptual definition of causation won't solve PoI.
\end{enumerate}

\subsection{Logic and Counterfactuals}\label{compare:logic}
In retrospect, "if-then" material conditionals cannot replace causality because it violates right-uniqueness: both $\{p:0, q:0\}$ and $\{p:0, q:1\}$ satisfy $p \Rightarrow q$. It allows vacuously true propositions like "if I don't eat anything today, then I am a billionaire," which feels wrong causally/counterfactually. By adding the laws-of-nature ($\M$) to the antecedents, we can perform rigorous deduction $\w=\texttt{VFI}(\M, \w_{exo})$ without sacrificing causal intuitions. The underlying causal formula is $q=f(p, \dots)$ instead of $p \Rightarrow q$, though functions and background conditions are often omitted in causal utterances.

Like the \textit{but-for} test, many counterfactual definitions of causation are variations of "if $x$, then $y$; if not-$x$, then not-$y$" \citep{sep-causation-counterfactual}. They're usually imperfect because they don't have the full expressiveness of functions. For example, \citet{mackie1965causes}'s INUS condition is equivalent to disjunctive normal form \citep{kim1971causes}, which any Boolean function can be converted to, so it's just a circuitous way of stating "causality is functions."

\citet{Lewis1973-LEWC-2} defines counterfactual conditional "if not-$x$, then not-$y$" as "in the closest possible world with not-$x$, there's not-$y$." However, without defining a distance metric and an algorithm to find the closest possible world, this definition cannot even describe a deterministic, fully-specified causal system. Using actual-tweaked contrast, SFM unambiguously computes the contrastive world as $\w_c=\texttt{CFI}(\M, \w_a, \w_{c|\C_{exo}})$.

\subsection{Intervention and SCM}\label{compare:scm}
The definition of SCM \citep{pearl2009causality} relies on intervention, a causal concept, so it's often criticized for being circular and non-reductive. We develop SFM as an equally-expressive reformulation of SCM that only relies on functions, thus eliminating circularity and providing a philosophical foundation for SCM. The generality of functions also avoids anthropocentric objections that manipulation requires human agency \citep{sep-causation-mani}.

Although SCM's surgical intervention $do(Y=y)$ is generalized by sub-SFM, we can also define it as a parent of $Y$, making intervention just a type of functional determination.

Given $\Pa(Y) = \{X_1, X_2, \dots, X_n, DoY\}$, $DoY$ is a \textit{surgical intervention} on $Y$ when:
\begin{enumerate}
\item $\D[DoY] = \D[Y] \cup \{\texttt{None}\}$. $\texttt{None} \notin \D[Y]$ means "no intervention," like in option types and nullable types.
\item There exists an "ordinary mechanism" function $g: \prod_{i=1}^{n} \D[X_i] \to \D[Y]$, such that
$$\F[Y](\w_{|\Pa(Y)}) = \begin{cases}
g(\w_{|\{X_1, X_2, \dots, X_n\}}) \textbf{ if } \w_{|\Pa(Y)}(DoY) = \texttt{None},\\
\w_{|\Pa(Y)}(DoY) \textbf{ otherwise}
\end{cases}
$$
\end{enumerate}

"Conditionally overriding an ordinary mechanism" is the key intuition behind interventions. For example,  barometer reading is ordinarily determined by atmospheric pressure, but it can also be manipulated by human intervention. SFM can also express more complicated interventions, like when intervention has a failure probability or when only some intervention options are possible for humans.

\section{Philosophical Applications}\label{apply:section}
Many philosophical discussions take "causation" as given without mathematically defining what it is, so our functional definition of causality may help clarify some downstream concepts.

\subsection{Desires for SFM Learning}
Several alleged "metaphysical doctrines" about causality can now be seen as epistemological desires for learning new SFMs:
\begin{enumerate}
\item The Principle of Sufficient Reason (PSR): "Everything has a cause" or "anything is an effect caused by earlier events" \citep{sep-sufficient-reason}.

PSR desires to add parents to exo-nodes that "have no causes" in old models.

\item The Eleatic Principle (EP): For something to "exist" in an ontology, it must be able to cause changes in other things \citep{colyvan1998eleatic}.

EP desires to add descendants to sink nodes that "affect nothing" in old models.

\item Causal Nexus (CN): "Any causal relation requires a nexus, some interface by means of which cause and effect are connected" \citep{sep-mental-causation}.

CN desires to insert intermediate nodes between old parent-child edges.
\end{enumerate}
Strictly speaking, these desires are not satisfiable if we only allow finite acyclic SFM (Appendix \ref{trilemma}), but they do encourage us to learn bigger SFMs to model the world.

\subsection{The Uncaused}\label{apply:uncaused}
Since exo-nodes can never appear in the "effect" part of causal utterances, we define node $u$ is \textit{uncaused} relative to $\M$ iff $u \in \V_{exo}$. Being uncaused/exogenous is not a metaphysical fact, but a modeling choice we make: We don't \textit{want} to model $u$ as being determined by a mechanism and other nodes in $\M$.

If $\M_{uni}$ is the SFM of the full world, we often only use some sub-SFM $\M_{sub}$ for specific tasks.
Because a node can be uncaused (exo-node) in one sub-SFM and caused (endo-node) in another, regarding "uncaused" as a node's metaphysical property without specifying $\M_{sub}$ is ill-defined. This is the source of many confusions.

For something with no causal parent anywhere, we say $u$ is \textit{strongly-uncaused} iff $u$ isn't an endo-node in any sub-SFM $\M_{sub}$ (i.e. it's an exo-node in $\M_{uni}$).

\subsection{Free Will}\label{apply:free}
Free will loosely describes an agent's ability to "freely" choose between different possible actions \citep{sep-freewill}. We often face seemingly conflicting intuitions:
\begin{enumerate}
\item People have free will.
\item The world's past and laws-of-nature functionally determine the world's future, making people's decisions unfree.
\end{enumerate} 

If we accept that something is \textit{free} if it's "uncaused or not deterministically caused" \citep{sep-causation-mani}, then SFM offers a mathematical definition of freedom that resolves this conflict:
\begin{enumerate}
\item Node $u$ is \textit{free} relative to $\M$ iff $u$ is exogenous in $\M$.
\item Node $u$ is \textit{unfree} relative to $\M$ iff $u$ is endogenous in $\M$.
\item Node $u$ is \textit{strongly-free} iff $u$ is exogenous in every $\M$ of interest that contains $u$.
\item Node $u$ is \textit{strongly-unfree} iff $u$ is endogenous in every $\M$ of interest that contains $u$.
\end{enumerate}

Whether an action is free depends on the model of interest.
When actions have consequences, we want to model the utility function $Q(s, a)$ for taking action $a \in A$ at state $s \in S$. This makes action free relative to $Q(s, a)$, so any action with consequences is not strongly-unfree. Meanwhile, the best action $a^*=\pi(s, Q)=\arg_{a \in A}\max Q(s, a)$ is determined/unfree relative to $\pi(s, Q)$. But for discrete $S, A$ without additional assumptions, finding the best action requires computing $Q(s, a)$ for all $a \in A$, so modeling $a$ as a "free" input is inevitable and useful: The agent evaluates the utility of each action before taking the best action.

Besides reinforcement learning \citep{sutton2018reinforcement}, this best-action-selection framework also applies to minimax search \citep{russell2010artificial} and decision-making in general. Although we don't define causality using agency like \citet{menzies1993causation}, we suggest that modeling "actions functionally determine consequences" could be an origin of human causal intuitions.

Generally, the freedom/arbitrariness/uncertainty of function inputs is closer to the universal quantifier "for all/any." It's not determined because we don't model it as another function's output; it's not random because we cannot reasonably specify its marginal distribution and even if we do, the distribution isn't helpful for the downstream task.

\begin{enumerate}
\item The light switch is free to vary, while the light is determined.
\item To maximize $f(x)$, we freely vary $x$ and record the maximum $f(x)$.
\item We freely change causes $\w_{1|\C_{exo}}$ and infer effects $\w_{1|\C_{endo}} \subseteq \texttt{CFI}(\M, \w_0, \w_{1|\C_{exo}})$.
\item A sorting algorithm works for an arbitrary input list.
\end{enumerate}

\subsection{Causal Explanation}\label{apply:explanation}
We use \textit{explanans} to explain \textit{explanandum}. A \textit{causal explanation} uses causes (and underlying mechanisms/laws-of-nature) to explain effects. There are also non-causal explanations that appeal to symmetric equations, correlation, or backtracking (using effects to explain causes).

With SFM, causal explanations become a subset of Deductive-Nomological (DN) explanations, where (1) explanans contains general laws and particular conditions; (2) explanandum is entailed by explanans \citep{sep-scientific-explanation}.

Using $\w_{a|\C_{exo}}$ to explain $\w_{a|\C_{endo}}$, we use general laws $\M$ and particular conditions $\w_c, \w_{a|exo}$; entailment comes from $\w_{a|\C_{endo}}\subseteq \w_{a}=\texttt{CFI}(\M, \w_c, \w_{a|\C_{exo}})$.
In practice, $\M$ can be learned from empirical data (inductive); some nodes' values can be random variables (probabilistic).

SFM solves many alleged counterexamples where DN model appears insufficient for defining explanation:
\begin{enumerate}
\item In the symmetric equation involving shadow length, the Sun's position, and flagpole height, why is shadow the explanandum? Because we prefer causal explanations over non-causal explanations and shadow should be modeled as a child node (Section \ref{learn:flagpole}).
\item Why do people omit irrelevant background conditions in explanations? Because we use delta compression in causal utterances (Section \ref{benchmark:bogus_prevent}, \ref{benchmark:irrelevant}).
\end{enumerate}
The asymmetry of causal explanation comes from the asymmetry of causality, which comes from functions being right-unique and often non-injective.

\subsection{Disposition}\label{apply:disposition}
Glass is fragile because it has a \textit{disposition} to shatter. Dispositions like fragility resemble properties of objects, but they describe possible (not necessarily actual) behaviors under certain conditions: Glass may not actually shatter \citep{sep-dispositions}. We analyze dispositions with functions.

For a deterministic and fully-specified example, minerals higher on Mohs hardness scale (e.g. diamond) will scratch softer minerals (e.g. talc). Let function $f(m_1, m_2)$ take in 2 minerals and return the mineral that gets scratched, so $\texttt{Talc} = f(\texttt{Diamond}, \texttt{Talc})$. Talc has the disposition to be scratched because $\forall m: \texttt{Talc} = f(m, \texttt{Talc})$; diamond has the "power" to scratch because $\forall m: m = f(m, \texttt{Diamond})$.

Therefore, dispositions are properties of a downstream function $f$, but people colloquially associate them with input nodes (scratch-hardness of minerals) or input values (scratch-hardness of diamond).

\subsection{Supervenience}\label{apply:supervenience}
"$Y$ \textit{supervenes} on $X$" is equivalent to "$Y$ functionally depends on $X$," because the formal definition of supervenience ("there cannot be an $Y$-difference without a $X$-difference" \citep{sep-supervenience}) is the same as right-uniqueness: Consider team $R$,
\begin{align*}
&\ \quad \neg ([(x, y_1)\in R] \land [(x, y_2) \in R] \land [y_1 \ne y_2])\\
&= \neg ([(x, y_1)\in R] \land [(x, y_2) \in R]) \lor \neg [y_1 \ne y_2]\\
&= \neg ([(x, y_1)\in R] \land [(x, y_2) \in R]) \lor [y_1 = y_2]\\
&= ([(x, y_1)\in R] \land [(x, y_2) \in R]) \Rightarrow [y_1 = y_2]
\end{align*}

\subsection{Mental Causation}\label{apply:mental}
Can mental kinds (property, state, event) cause physical kinds? Mental causation faces 2 conflicting intuitions:
\begin{enumerate}
\item It's common in everyday experiences: I want to raise my hand (mental state), so I raise my hand (body state).
\item The \textit{Exclusion Problem}: Physical effects like body movements are already determined by physical causes like brain activities, so there's no room for a mental cause, which is also sufficient for the physical effect  \citep{sep-mental-causation}.
\end{enumerate}

The Exclusion Problem arises whenever a functional determination entailed by $(\V, \E, \D, \F)$ cannot be deduced from the graph structure $\G=(\V, \E)$ (and Armstrong's Axioms) alone. People feel uneasy because they cannot find such dependency as a path in $\G$.

\begin{enumerate}
\item (Assumption) On the lowest physical level, brain state functionally determines body state: $\texttt{BrainState} \xrightarrow{f_1} \texttt{BodyState}$.
\item (Assumption) Mental state supervenes on brain state: $\texttt{BrainState} \xrightarrow{f_2} \texttt{MentalState}$. Mental state is an abstract/aggregate description of physical brain state.

\textit{Multiple realizability} (a single mental kind can be realized by many distinct physical kinds) \citep{sep-multiple-realizability} is true when $f_2$ is non-injective.

It's a coincidence that \textit{functionalism} (name unrelated to mathematical functions) uses causality to define mental states \citep{sep-functionalism} and we reduce causality to functions.
\item So we have an SFM $\M$ with graph $\texttt{BodyState} \leftarrow \texttt{BrainState} \to \texttt{MentalState}$ and functions $\F = \{\texttt{BodyState}: f_1, \texttt{MentalState}: f_2\}$.
\item (Fact) There exists a function $f_3$ such that $\texttt{MentalState} \xrightarrow{f_3} \texttt{BodyState}$ is true in every $\w \in R_{\M}$. Mental state does functionally determine body state.

$\texttt{MentalState} \xrightarrow{f_3} \texttt{BodyState}$ cannot be deduced from $\G$ alone. It's entailed by the specific functional mappings $f_1, f_2$ (and $(\V, \E, \D)$). Although it cannot appear as a path in $\G$, we see no reason to dismiss it as "excluded." We may use SFM-intersection to explicitly/graphically encode this functional dependency, although $R_{\M}$ is entailed by a single SFM $\M$ (SFM-intersection-proper isn't required).
\end{enumerate}

The Exclusion Problem appears in any system with hierarchical levels of abstraction, since supervenience is just functional dependency. In fully-specified and deterministic computers, what causes a video to play on screen, the low-level chip activities or the high-level video-player program? The same reasoning applies. The only empirical question is whether higher-level functional determinations like $f_3$ are true. If not, we simply say the abstraction is broken.

\subsection{Time}\label{apply:time}
SFM doesn't endorse any particular theory of time, but we can define $T: \V \to \mathbb{R}$ that maps each node $u$ to a real-valued timestamp $T(u)$. If $\forall (u, v) \in \E: T(u) \le T(v)$, then causes always temporally precede their effects. But without additional assumptions, $T$ might as well violate this condition.

\textit{Backward causation} occurs when an effect temporally precedes its cause \citep{sep-causation-backwards}. If most SFM edges $(u, v) \in \E$ still point from past to future ($T(u) \le T(v)$), a backward edge can create cycles, resulting in PULO or actually unsatisfiable laws. That's why people intuitively dislike backward causation. But if the specific SFM is satisfiable or satisfied by empirical data, we cannot dismiss it \textit{a priori}.

Are there fundamental properties of our physical world that make causal and temporal orders agree? Could it be the asymmetry of thermodynamics, radiation \citep{sep-causation-physics}, or our mental habit of "actions determining consequences" (Section \ref{apply:free})? Further research is required.

In a causal feedback loop $A \to A$, node $A$ influences its own next state. With discrete time, we can unroll it to an acyclic time-indexed causal chain $A(0) \xrightarrow{\F[A(1)]} A(1) \xrightarrow{\F[A(2)]} A(2) \xrightarrow{\F[A(3)]} \dots$, which may be countably infinite. When every $\F[A(t)]$ is invertible, the system has time-symmetry and another equivalent SFM $A(0) \xleftarrow{\F[A(1)]^{-1}} A(1) \xleftarrow{\F[A(2)]^{-1}} A(2) \xleftarrow{\F[A(3)]^{-1}} \dots$.

Decreasing the interval between 2 consecutive timestamps towards the infinitesimal, we eventually get an uncountable number of nodes and cannot properly define an edge, because there are no 2 "consecutive" real numbers. In this case, it would make more sense for $A(t)$ to determine its instantaneous rate of change $\frac{d}{dt} A(t)$, like the exponential/logistic growth rate of bacteria population size and the predator-prey dynamics in Lotka-Volterra equations. SFM-intersection-proper can represent autonomous differential equations, which include causal loop diagrams \citep{haraldsson2004introduction}. However, differential equation in general are better tools for modeling continuous-time causality.

\section{Conclusion}\label{conclusion}
After our conceptual analysis that reduces causality to functions, there should be nothing mysterious about the definition of causality.
Using forward inference, contrast, and delta compression, Structural Functional Model (SFM) correctly produces intuitive causal utterances.
We've also supported intuitive practices from an algorithmic perspective: contrast saves space and time; finite acyclic SFM is required for guaranteed satisfiability (at the cost of expressiveness).
Distinct from but compatible with probability theory, "causality as functions" allows for interesting downstream applications.

\bibliography{causality}

\begin{appendices}
\section{Mathematics Review}\label{math_foundations}
Under ZFC set theory, a \textit{set} is roughly an unordered collection of distinct elements. The binary \textit{Cartesian product} between two sets $X$ and $Y$ is $X \times Y = \{(x, y)|x \in X \land y \in Y\}$.
A binary \textit{relation} $R$ over $X$ and $Y$ is $R \subseteq X \times Y$.
A relation $R$ may have properties:
\begin{enumerate}
\item Left-total: $\forall x \in X \exists y \in Y: (x,y) \in R$
\item Right-total: $\forall y \in Y \exists x \in X: (x,y) \in R$
\item Left-unique: $\forall x_1 \in X, x_2 \in X, y \in Y: ((x_1,y)\in R) \land ((x_2,y) \in R) \Rightarrow x_1=x_2$
\item \textit{Right-unique}: $\forall x \in X, y_1 \in Y, y_2 \in Y: ((x,y_1)\in R) \land ((x,y_2)\in R) \Rightarrow y_1=y_2$
\item \textit{Function} (total function): left-total and right-unique.
\item Partial function: right-unique.
\item \textit{Injective function}: left-unique function.
\item Surjective function: right-total function.
\item Bijective function: injective and surjective function.
\end{enumerate}
Because of right-uniqueness, a function can be written as $f: X \to Y$ such that $f(x) \in Y$ is unique for every $x \in X$. For functions $f: X \to Y$ and $g: S \to Y$ satisfying $S \subseteq X$, if $\forall x \in S: f(x)=g(x)$, we say $g$ is a \textit{restriction} of $f$ and $f$ is an \textit{extension} of $g$ (or $f$ extends $g$). Because functions are relations, we write $g \subseteq f$ or $g = f_{|S}$.

An \textit{indexed collection of sets }is a 3-tuple $(I, \mathcal{A}, A)$ written as $\{A_i\}_{i\in I}$, where $I$ is the index set, $\mathcal{A}$ is a collection of sets, and $A$ is a function $A: I \to \mathcal{A}$. Every $A_i = A(i) \in \mathcal{A}$ is a set. Now we can define Cartesian product over any (possibly infinite-sized) indexed collection of sets: $\prod_{i\in I}A_{i}$ is the set of all functions $f: I \to \bigcup_{i\in I}A_{i}$ such that $\forall i\in I: f(i) \in A_{i}$.
Similarly, a relation over an indexed collection of sets is a subset of its Cartesian product.

A \textit{directed graph} is an ordered pair $\G = (\V, \E)$, where $\V$ is a set of nodes and $\E \subseteq \V \times \V$ is a set of directed edges. A \textit{directed edge} is an ordered pair $(u, v)$ such that $u \in \V$ and $v \in\V$.
\begin{itemize}
\item If $(u, v) \in \E$, $u$ is a \textit{parent} of $v$ and $v$ is a \textit{child} of $u$.
\item $\Pa(u)$ denotes the set of parents of $u$; $\mathrm{Ch}(u)$ denotes the set of children of $u$.
\item The \textit{indegree} of $u$ is the number of its parents ($\mathrm{deg}^{-}(u) = |\Pa(u)|$); the \textit{outdegree} of $u$ is the number of its children ($\mathrm{deg}^{+}(u) = |\mathrm{Ch}(u)|$); the \textit{degree} of $u$ is the sum of its indegree and outdegree ($\mathrm{deg}(u)=\mathrm{deg}^{-}(u)+\mathrm{deg}^{+}(u)$).
\item A \textit{root} node $u$ has indegree $\mathrm{deg}^{-}(u) = 0$. A \textit{sink} node $u$ has outdegree $\mathrm{deg}^{+}(u) = 0$.
\item A \textit{path} is a sequence of nodes $v_1, v_2, \dots, v_n$ such that $(v_i, v_{i+1}) \in \E$ for any $i \in \{1, 2, \dots, n-1\}$; a path is a \textit{cycle} if $v_1 = v_n$.
\item Node $u$ is an \textit{ancestor} of node $v$ ($u \in \mathrm{An}(v)$) if there's a path from $u$ to $v$. Otherwise, $u$ is a \textit{non-ancestor} of $v$.
\item Node $u$ is a \textit{descendant} of node $v$ ($u \in \mathrm{De}(v)$) if there's a path from $v$ to $u$. Otherwise, $u$ is a \textit{non-descendant} of $v$.
\item Under our convention, a node is the ancestor/descendant of itself.
\end{itemize}

\section{Generalized M\"unchhaussen Trilemma}\label{trilemma}
We formalize a theorem that generalizes M\"unchhaussen Trilemma \citep{sep-formal-epistemology} in epistemology:

\begin{theorem}[\textbf{Generalized M\"unchhaussen Trilemma (GMT)}]
Any directed graph $G = (V, E)$ contains at least one of the following:
\begin{itemize}
\item A root.
\item A cycle.
\item An infinite regress: An infinite path of distinct nodes $(\dots, u_2, u_1, u_0)$ ending at $u_0$, such that for any integer $i \ge 1$, there exists $u_i \in V$ satisfying $(u_i, u_{i-1}) \in E$ and $u_i \notin \{u_j\}_{j=0}^{i-1}$
\end{itemize}
\end{theorem}

\begin{proof}
Proving by contradiction, suppose instead that a graph has no root, no cycle, and no infinite regress. Since there's no infinite regress, there exists a nonnegative integer $n$ that is the maximum length of a path of distinct nodes ending at some $u \in V$. Let $(u_n, u_{n-1}, \dots, u_1, u_0)$ be that maximum-length path.

Because $G$ doesn't have a root, $\mathrm{deg}^-(u) \ge 1$ for all $u \in V$ and thus $\mathrm{deg}^{-}(u_n) \ge 1$, so there exists a node $v \in V$ such that $(v, u_n) \in E$ is an edge.

If $v \in \{u_i\}_{i=0}^{n}$, then $v = u_i$ for some integer $0 \le i \le n$. We can construct a new path $(u_n, u_{n-1}, \dots, u_i, u_n)$. It's a path because $(u_i, u_n) = (v, u_n) \in E$ and $(u_i, u_{i-1}) \in E$ for all $1 \le i \le n$; it's a cycle because it starts and ends at $u_n$. This violates the acyclic assumption, so $v \notin \{u_i\}_{i=0}^{n}$ is distinct from all nodes in the path.

We can thus construct a new path $(v, u_n, u_{n-1}, \dots, u_1, u_0)$ of length $n+1$, where all nodes have been shown to be distinct. However, this contradicts the condition that the maximum length of distinct-node paths is $n$. Therefore, it's impossible for a directed graph to have no root, no cycle, and no infinite regress at the same time.
\end{proof}




GMT is a general theorem about directed graphs, proven mathematically. It applies to all problems characterized by objects and directed binary relations between them (i.e. describable by a directed graph), such as "$X$ causes $Y$" and "$X$ justifies $Y$."

If we define a directed graph where nodes are propositions and edge $(u, v)$ means "$u$ justifies $v$" or "$u$ is a part of the justification for $v$", then GMT entails that we either settle with foundationalism (root that isn't justified), coherentism (cycle that justifies itself), or infinitism (infinite regress of justification chain) - we cannot simultaneously eliminate all 3 of them. Notice how we never used the meaning of "justification" in our proof, only the directed binary form of "$A$ justifies $B$."

\end{appendices}

\end{document}